\documentclass{article} 

\usepackage{amsmath,amsfonts,bm}

\def\eqref#1{equation~\ref{#1}}

\def\1{\bm{1}}

\DeclareMathAlphabet{\mathsfit}{\encodingdefault}{\sfdefault}{m}{sl}
\SetMathAlphabet{\mathsfit}{bold}{\encodingdefault}{\sfdefault}{bx}{n}

\newcommand{\E}{\mathbb{E}}

\newcommand{\Var}{\mathrm{Var}}

\DeclareMathOperator*{\argmin}{arg\,min}

\PassOptionsToPackage{numbers, compress}{natbib}
\usepackage{natbib}
\usepackage{fullpage}
\usepackage[utf8]{inputenc} 
\usepackage[T1]{fontenc}    
\usepackage{hyperref}    
\usepackage{url}
\usepackage{enumitem}
\usepackage{microtype} 
\usepackage{amsfonts} 
\usepackage{graphicx}
\usepackage{subfigure}
\usepackage{booktabs} 
\usepackage{multirow,xcolor}
\usepackage{amsthm}
\usepackage{amsmath}
\usepackage{stmaryrd}  
\usepackage{bbm}
\usepackage{amssymb}
\usepackage{wrapfig}
\newtheorem{theorem}{Theorem}

\newtheorem{lemma}{Lemma}

\usepackage{algorithm}
\usepackage{algorithmic}
\usepackage{nicefrac}

\newcommand{\merlin}{{MeRLin}}
\newcommand{\merlinfinetune}{{MeRLin-ft}}
\newcommand{\targetonly}{{target-only}}
\newcommand{\Targetonly}{{Target-only}}
\newcommand{\thetahatphi}{\widehat{\theta}_t(\phi)}
\newcommand{\losstarget}{L_{\textup{meta},t}}
\newcommand{\lossmeta}{L_{\textup{meta}}}
\newcommand{\lossjoint}{L_{\textup{joint}}}
\newcommand{\raw}{$\rightarrow$}

\newcommand{\Exp}{\mathrm{\mathbb{E}}}
\newcommand{\Real}{\mathbb{R}}
\newcommand{\norm}[1]{\left\lVert#1\right\rVert}
\newcommand{\lbr}{\left(}
\newcommand{\rbr}{\right)}
\newcommand{\lbbr}{\left[}
\newcommand{\rbbr}{\right]}

\newcommand{\loss}[1]{L_{#1}}
\newcommand{\losssub}[2]{\ell_{{#1}}(#2)}

\newcommand{\Lmeta}[2]{\lossmeta^{#1}({#2})}
\newcommand{\Lsource}[2]{L^{#1}_{\DS}({#2})}
\newcommand{\Ltarget}[2]{L^{#1}_{\DTe}({#2})}
\newcommand{\Ljoint}[2]{\lossjoint^{#1}({#2})}
\newcommand{\LossD}[2]{L_{#1}({#2})}

\newcommand{\w}[1]{\phi_{{#1}}}
\newcommand{\W}{\phi}
\newcommand{\agen}{{\theta}}
\newcommand{\as}{{\theta_s}}
\newcommand{\at}{{\theta_t}}
\newcommand{\df}{{m}}
\newcommand{\DS}{\mathcal{D}_s}
\newcommand{\DT}{\mathcal{D}_t}
\newcommand{\DSe}{\widehat{\mathcal{D}}_s}
\newcommand{\DTe}{\widehat{\mathcal{D}}_t}
\newcommand{\DTa}{\widehat{\mathcal{D}}_{t}^{\textup{tr}}}
\newcommand{\DTb}{\widehat{\mathcal{D}}_{t}^{\textup{val}}}
\newcommand{\thetapre}{\hat{\theta}_s}
\newcommand{\phipre}{\hat{\phi}_{\textup{pre}}}
\newcommand{\phimeta}{\hat{\phi}_{\textup{meta}}}
\newcommand{\phijoint}{\hat{\phi}_{\textup{joint}}}
\newcommand{\nt}{{n_t}}
\newcommand{\thetahatphimeta}{\hat{\theta}_t(\hat{\phi}_{\textup{meta}})}

\graphicspath{{figures/}}

\title{Meta-learning Transferable Representations with a Single Target Domain}
\author{Hong Liu \thanks{Tinghua University, email: \texttt{h-l17@mails.tsinghua.edu.cn}} \and Jeff Z. HaoChen \thanks{Stanford University, email: \texttt{jhaochen@stanford.edu}}  \and Colin Wei \thanks{Stanford University, email: \texttt{colinwei@stanford.edu}} \and Tengyu Ma \thanks{Stanford University, email: \texttt{tengyuma@stanford.edu}}}

\begin{document}
\maketitle

\begin{abstract}
	Recent works found that fine-tuning and joint training---two popular approaches for transfer learning---do not always improve accuracy on downstream tasks. First, we aim to understand more about when and why fine-tuning and joint training can be suboptimal or even harmful for transfer learning. We design semi-synthetic datasets where the source task can be solved by either source-specific features or transferable features. 
	We observe that (1) pre-training may not have incentive to learn transferable features  and (2) joint training may simultaneously learn source-specific  features and  overfit to the target. 
Second, to  improve over fine-tuning and joint training, we propose \textbf{Me}ta \textbf{R}epresentation \textbf{L}earn\textbf{in}g  (\merlin) to learn transferable features.~\merlin~meta-learns representations by ensuring that a head fit on top of the representations with target training data also performs well on target validation data. We also prove that \merlin~recovers the target ground-truth model with a quadratic neural net parameterization and a source distribution that contains both  transferable and source-specific features. On the same distribution,  pre-training and joint training provably fail to learn transferable features. {\merlin} empirically outperforms previous state-of-the-art transfer learning algorithms  on various real-world vision and NLP transfer learning benchmarks. 
\end{abstract}

\section{Introduction}
 
Transfer learning---transferring knowledge learned from a large-scale source dataset to a small target dataset---is an important paradigm in machine learning~\citep{cite:NIPS14CNN} with wide applications in computer vision~\citep{pmlr-v32-donahue14} and natural language processing (NLP)~\citep{2018Universal,devlin-etal-2019-bert}. Because the source and target tasks are often related, we expect to learn features that are transferable to the target task from the source data. These features may help learn the target task with fewer examples \citep{cite:ICML15DAN,tamkin2020investigating}.

Mainstream approaches for transfer learning are fine-tuning and joint training. Fine-tuning initializes from a model pre-trained on a large-scale source task (e.g., ImageNet) and  continues training  on the target task with a potentially different set of labels (e.g., object recognition~\citep{Wang_2017_CVPR,2018Learning,alex2019big}, object detection~\citep{Girshick_2014_CVPR}, and segmentation \citep{7298965,He_2017_ICCV}). Another enormously successful example of fine-tuning is in NLP: pre-training transformers and fine-tuning on downstream tasks leads to state-of-the-art results for many NLP tasks~\citep{devlin-etal-2019-bert,NIPS2019_8812}. In contrast to the two-stage optimization process of fine-tuning, joint training optimizes a linear combination of the objectives of the source and the target tasks~\citep{Kokkinos_2017_CVPR,2017Multi,liu-etal-2019-multi-task}. 
 
Despite the pervasiveness of fine-tuning and joint training, recent works uncover that they are not always panaceas for transfer learning. \citet{2018ImageNet} found that the pre-trained models learn the texture of ImageNet, which is biased and not transferable to target tasks.
ImageNet pre-training does not necessarily improve accuracy on COCO~\citep{cite:arxivrethink}, fine-grained classification~\citep{Kornblith_2019_CVPR}, and medical imaging tasks \citep{NIPS2019_8596}. \citet{Zhang_2020_ICLR} observed that large model capacity and discrepancy between the source and target domain eclipse the effect of joint training. Nonetheless, we do not yet have a systematic understanding of what makes the successes of fine-tuning and joint training inconsistent. 

The goal of this paper is two-fold: (1) to understand more about when and why fine-tuning and joint training can be suboptimal or even harmful for transfer learning; (2) to design algorithms that overcome the drawbacks of fine-tuning and joint training and consistently outperform them.

To address the first question, we hypothesize that fine-tuning and joint training do not have incentives to prefer learning transferable features over source-specific features, and thus whether they learn transferable features is rather coincidental and depends on the property of the datasets.  To empirically analyze the hypothesis, we design a semi-synthetic dataset that contains artificially-amplified transferable features and source-specific features simultaneously in the source data. Both the transferable and source-specific features can solve the source task, but only transferable features are useful for the target. We analyze what features fine-tuning and joint training will learn. See Figure \ref{fig:semi} for an illustration of the semi-synthetic experiments. We observed following failure patterns of fine-tuning and joint training on the semi-synthetic dataset.

\begin{itemize}[leftmargin=*]
\item {
	Pre-training may learn non-transferable features that don't help the target when both transferable and source-specific features can solve the source task, since it’s oblivious to the target data. When the dataset contains source-specific features that are more convenient for neural nets to use, pre-training learns them; as a result, fine-tuning starting from the source-specific features does not lead to improvement.	}

\item {
	Joint training learns source-specific features and overfits on the target.
	A priori, it may appear that the joint training should prefer transferable features because the target data is present in the training loss. However, joint training easily overfits to the target especially when the target dataset is small. 
When the source-specific features are the most convenient for the source, joint training simultaneously learns the source-specific features  and memorizes the target dataset. 
}
\end{itemize}

Toward overcoming the drawbacks of fine-tuning and joint training, we first note that any proposed algorithm, unlike fine-tuning, should use the source and the target simultaneously to encourage extracting shared structures. Second and more importantly, we recall that good representations should enable generalization: we should not only be able to fit a target head with the representations (as joint training does), but the learned head should also generalize well to a held-out target dataset. 
With this intuition, we propose \textbf{Me}ta \textbf{R}epresentation \textbf{L}earn\textbf{in}g (\textbf{\merlin}) to encourage learning transferable and generalizable features: we meta-learn a feature extractor such that the head fit to a target training set performs well on a target validation set. In contrast to the standard model-agnostic meta-learning (MAML) \citep{pmlr-v70-finn17a}, which aims to learn prediction models that are adaptable to multiple target tasks from multiple source tasks, our method meta-learns transferable representations with only \textit{one source and one target domain}. 

Empirically, we first verify that \merlin~learns transferable features on the semi-synthetic dataset. We then show that \merlin~outperforms state-of-the-art transfer learning baselines in real-world vision and NLP tasks such as ImageNet to fine-grained classification and language modeling to GLUE. 

Theoretically, we analyze the mechanism of the improvement brought by \merlin. In a simple two-layer quadratic neural network setting,  we prove that \merlin~recovers the target ground truth with only limited target examples whereas both fine-tuning and joint training fail to learn transferable features that can perform well on the target. 

In summary, our contributions are as follows. (1) Using a semi-synthetic dataset, we analyze and diagnose when and why fine-tuning and joint training fail to learn transferable representations. (2) We design a meta representation learning algorithm ({\merlin}) which outperforms state-of-the-art transfer learning baselines. (3) We rigorously analyze the behavior of fine-tuning, joint training, and {\merlin} on a special two-layer neural net setting. 
\section{Setup and Preliminaries}\label{sec:preliminaries}

In this paper, we study \textit{supervised transfer learning}. Consider an input-label pair $(x,y)\in\Real^d\times\Real$. We are provided with a source distributions $\DS$ and a target distribution $\DT$ over $\Real^d\times\Real$. The source dataset $\DSe=\{x_{i}^{s},y_{i}^{s}\}_{i=1}^{n_s}$ and the target dataset $\DTe=\{x_{i}^{t},y_{i}^{t}\}_{i=1}^{n_t}$ consist of $n_s$ i.i.d. samples from $\DS$ and $n_t$ i.i.d. samples from $\DT$ respectively. 
Typically $n_s\gg n_t$. 
We view a predictor as a composition of a feature extractor $h_{\phi}:\Real^d\rightarrow\Real^m$ parametrized by $\phi\in\boldsymbol{\Phi}$, which is often a deep neural net, and a head classifier $g_{\theta}:\Real^m\rightarrow\Real$ parametrized by $\theta\in\boldsymbol{\Theta}$, which is often linear. That is, the final prediction is $f_{\theta,\phi}(x)=g_{\theta}(h_{\phi}(x))$. 
Suppose the loss function is $\ell(\cdot,\cdot)$, such as cross entropy loss for classification tasks. Our goal is to learn an accurate model on the target domain $\DT$. 

\begin{figure}[t]
  \centering
   \subfigure{
    \includegraphics[width=0.69\textwidth]{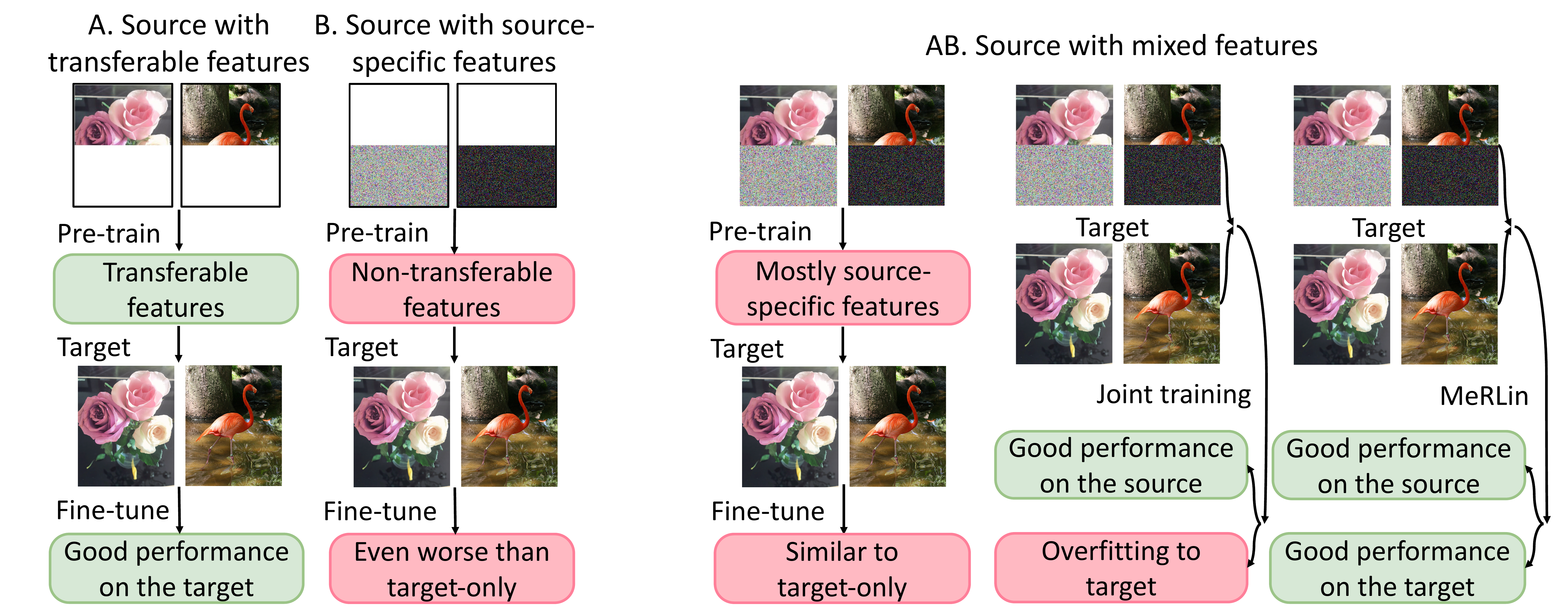} 
    }
    \hspace{10pt}
   \subfigure{
    \includegraphics[width=0.24\textwidth]{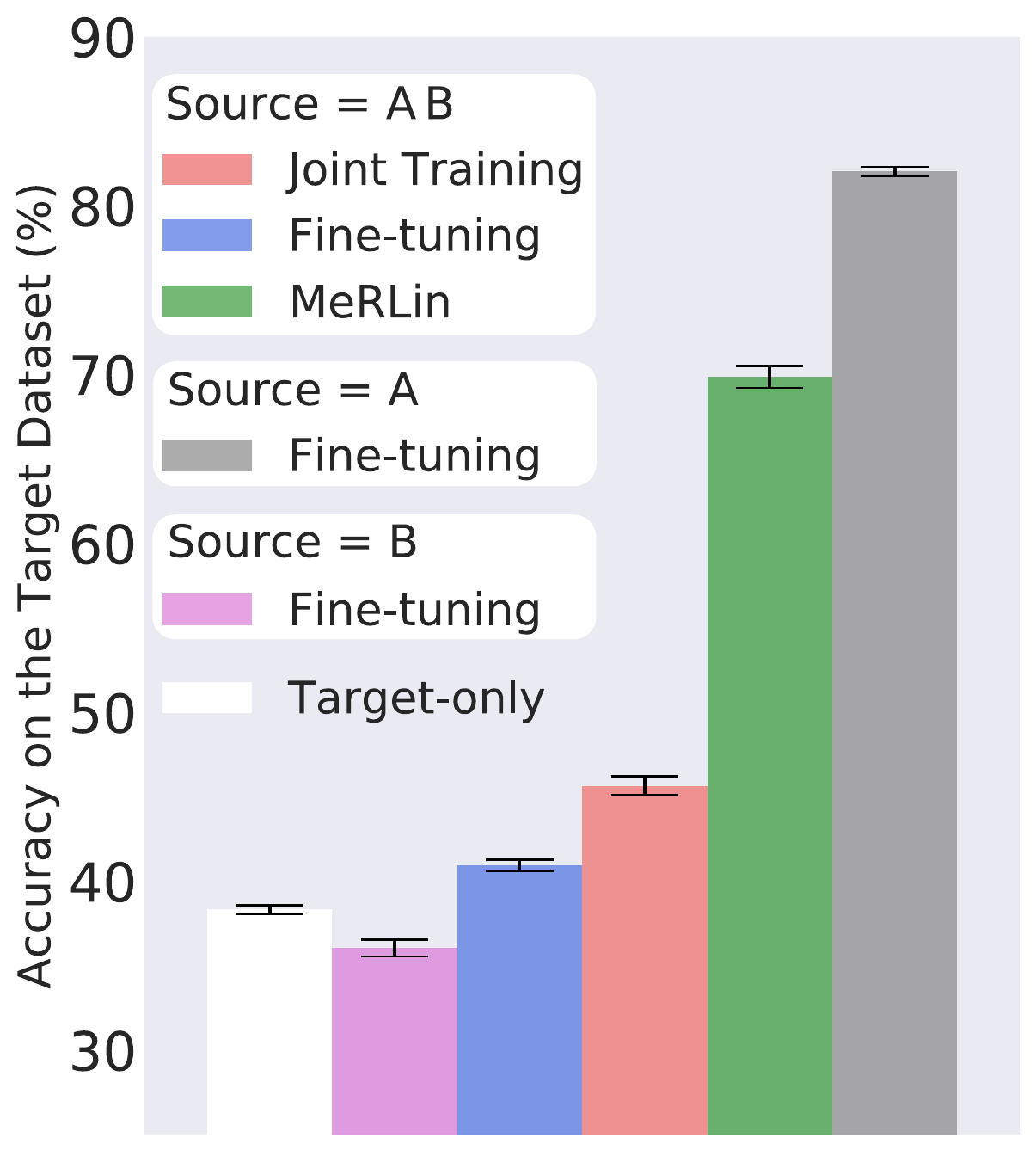}
    }
  \caption{    \label{fig:semi} Illustration of the features learned by fine-tuning, joint training, and {\merlin} on the semi-synthetic dataset. \textbf{Left}: The semi-synthetic dataset and the qualitative observations on the representations learned by three algorithms. \textbf{Right:} Quantitative results on the target test accuracy. See more interpretations, analysis, and results in Section~\ref{sec:limitations}.} 
\end{figure}

Since the label sets of the source and target tasks can be different, we usually learn two heads for the source task and the target task separately, denoted by $\theta_s$ and $\theta_t$, with a shared feature extractor $\phi$. Let  $L_{\widehat{D}}(\theta, \phi)$ be the empirical loss of model $g_{\theta}(h_{\phi}(x))$ on the empirical distribution $\widehat{D}$, that is, $L_{\widehat{D}}(\theta, \phi) := \E_{(x,y)\in \widehat{D}} \ell(g_{\theta}(h_\phi(x)),y)$ where $(x,y)\in \widehat{D}$ means sampling uniformly from the dataset $\widehat{D}$. Using this notation, the standard supervised loss on the source (with the source head $\theta_s$) and loss on the target (with the target head $\theta_t$) can be written as $
L_{\DSe}(\theta_s, \phi)$ and $L_{\DTe}(\theta_t, \phi)$ respectively.

We next review mainstream transfer learning baselines and describe them in our notations. 

\noindent\textbf{\Targetonly} is the trivial algorithm that only trains on the target data $\DTe$ with the objective $\loss{\DTe}(\theta_t,\phi)$ starting from random initialization. 
With insufficient target data, target-only is prone to overfitting.

\noindent\textbf{Pre-training} starts with random initialization and \textit{pre-trains} on the source dataset with objective function $\loss{\DSe}(\theta_s,\phi)$ to obtain the pre-trained feature extractor $\phipre$ and head $\thetapre$. 

\noindent\textbf{Fine-tuning} initializes the target head $\theta_t$ randomly and initializes the feature extractor $\phi$ by $\phipre$ obtained in pre-training, and \textit{fine-tunes} $\phi$ and $\theta_t$ on the target by optimizing $\loss{\DTe}(\theta_t,\phi)$ over both $\theta_t$ and $\phi$. Note that in this paper, fine-tuning refers to fine-tuning all layers by default. 

\noindent\textbf{Joint training} starts with random initialization, and trains on the source and target dataset jointly by optimizing a linear combination of their objectives over the heads $\theta_s$, $\theta_t$ and the shared feature extractor $\phi$: $\min_{\theta_s, \theta_t, \phi}~ \Ljoint{}{\as,\at,\W} := (1-\alpha) \loss{\DSe}(\theta_s, \phi) + \alpha \loss{\DTe}(\theta_t, \phi)$. The hyper-parameter $\alpha$ is used to balance source training and target training. We use cross-validation to select optimal $\alpha$.
\section{Limitations of Fine-tuning and Joint Training: Analysis on Semi-synthetic Data}\label{sec:limitations}

Previous works~\citep{cite:arxivrethink, Zhang_2020_ICLR} have observed cases when fine-tuning and joint training fail to improve over target-only. Our hypothesis is that both pre-training and joint training do not have incentives to prefer learning transferable features over source-specific features, and thus the performance of fine-tuning and joint training rely on whether the transferable features happen to be the best features for predicting the source labels. Validating this hypothesis on real datasets is challenging, if not intractable---it's unclear what's the precise definition or characterization of transferable features and source-specific features. Instead, we create a semi-synthetic dataset where transferable features and source-specific features are prominent and well defined. 

\paragraph{A semi-synthetic dataset.} The target training dataset we use is a uniformly-sampled subset of the CIFAR-10 training set of size 500. The target test dataset is the original CIFAR-10 test set. The source dataset of size 49500, denoted by AB, is created as follows. The upper halves of the examples are the upper halves of the CIFAR-10 images (excluding the 500 example used in target). The lower halves contain a signature pattern that strongly correlates with the class label: for class $c$, the pixels of the lower half are drawn i.i.d. from gaussian distribution $\mathcal{N}(c/10,0.2^2)$. Therefore, averaging the pixels in the lower half of the image can reveal the label because the noise will get averaged out.  The benefit of this dataset is that any features related to the top half of the images can be defined as transferable features, whereas the features related to the bottom half are source-specific. Moreover, we can easily tell which features are used by a model by testing the performance on images with masked top or bottom half. For analysis and comparison, we define A to be the dataset that contains the top half of dataset AB and zeros out the bottom half, and B vice versa.  See Figure~\ref{fig:semi} (left) for an illustration of the datasets. Further details are deferred to Section \ref{sec:semi_details}

\begin{figure}[t]
  \centering
  \subfigure[T-SNE embeddings of features on the target dataset.]{
    \includegraphics[width=0.634\textwidth]{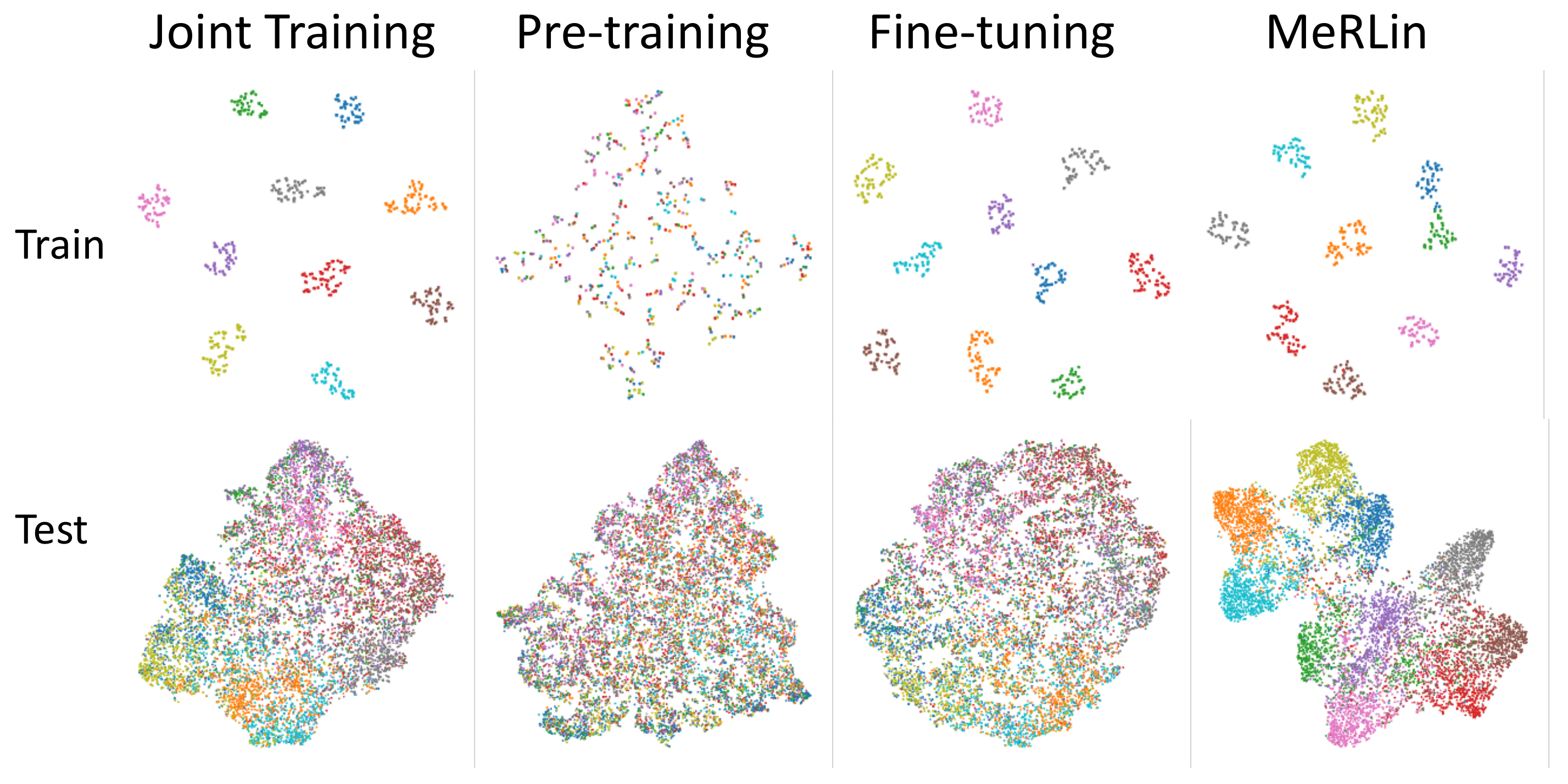}
    \label{fig:tsne}
    }
    \hspace{10pt}
      \subfigure[Ablation.]{
    \includegraphics[width=0.225\textwidth]{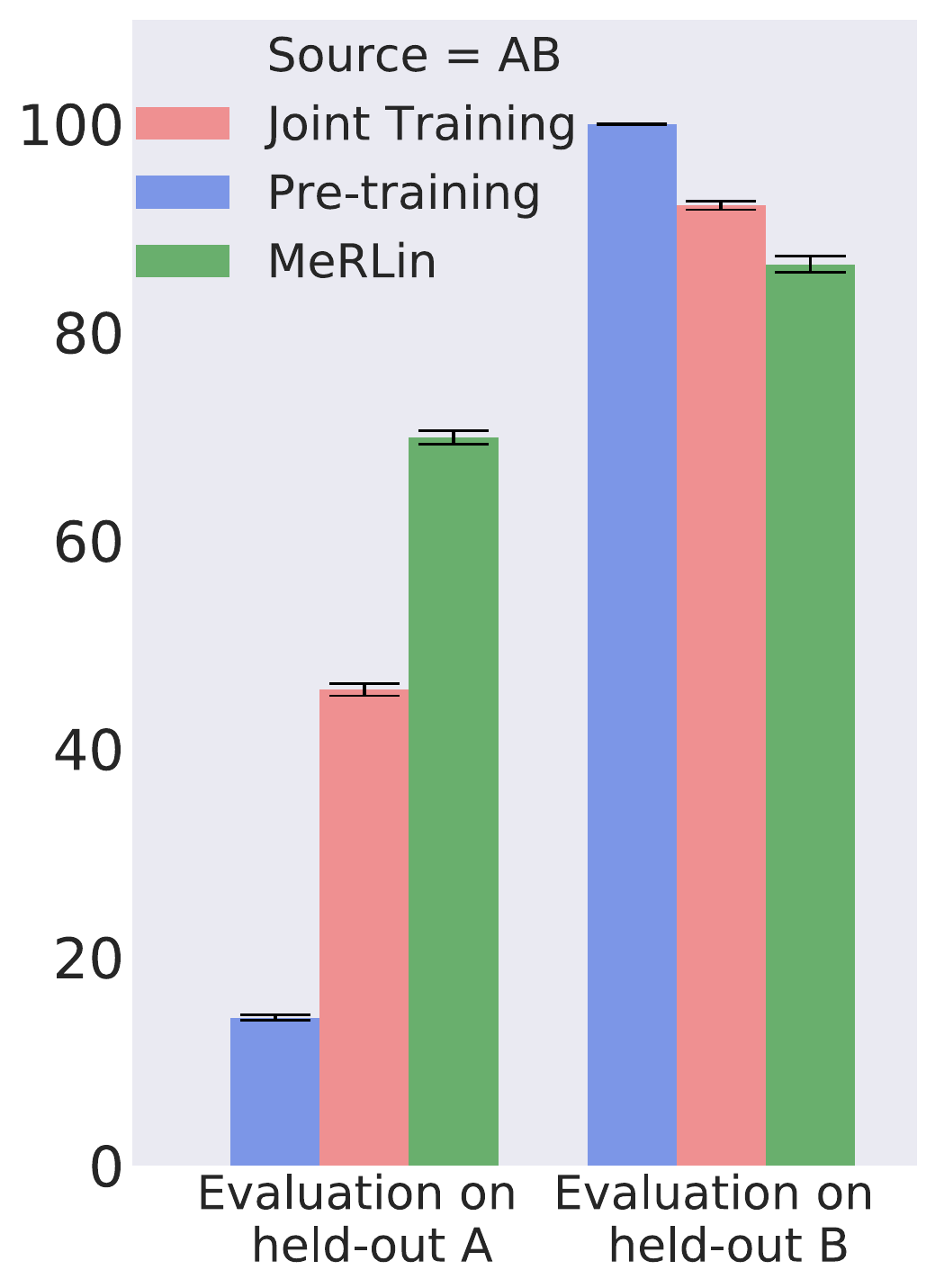}
    \label{fig:semi2}
    }
 \caption{(a) \textbf{T-SNE visualizations of features on the target train and test set}. The representations of pre-training work poorly on both target train and test set, indicating that transferable features are not learned. Both joint training and fine-tuning work well on the target train set but poorly on the test set, indicating overfitting. \merlin~works well on the target test set. (b) \textbf{Evaluation of different methods on A and B.} Joint-training and pre-training rely heavily on the source-specific feature B and learn the transferable feature A poorly compared to~\merlin. See more details in Section~\ref{sec:limitations}.}
  \label{fig2}
\end{figure}

In Figure \ref{fig:semi} (right), we evaluate various algorithms' performance on target test data. In Figure~\ref{fig:tsne} (left), we run algorithms with AB being the source dataset and  visualize the learned features on the target training dataset and target test dataset to examine the generalizability of the features. In Figure~\ref{fig:tsne} (right), we evaluate the algorithms on the held-out version of dataset A and B to examine what features the algorithms learn. ResNet-32 \citep{cite:CVPR16DRL} is used for all settings. 

\paragraph{Analysis: }  First of all, \textit{target-only} has low accuracy (38\%) because the target training set is small. 
Except when explicitly mentioned, all the discussions below are about algorithms on the source AB. 

\textit{Fine-tuning fails because pre-training does not prefer to learn transferable features and fine-tuning overfits.} Figure~\ref{fig:semi2} (pre-training) shows that the pre-trained model has near-trivial accuracy on held-out A but near-perfect accuracy on held-out B, indicating that it solely relies on the source-specific feature (bottom half) and does not learn transferable features. Figure~\ref{fig:tsne} (pre-training) shows that indeed pre-trained features do not have even correlation with target training and test sets. Figure~\ref{fig:tsne} (fine-tuning) shows that fine-tuning improves the features' correlation with the training target labels but it does not generalize to the target test because of overfitting. The performance of fine-tuning (with source =AB) in Figure~\ref{fig:semi} (right) also corroborates the lack of generalization. 

\textit{Joint training fails because it simultaneously learns mostly source-specific features and features that overfit to the target. } Figure~\ref{fig:semi2} (joint training) shows that the joint training model performs much better on held-out B (with 92\% accuracy) than on the held-out A (with 46\% accuracy), indicating it learns the source-specific feature very well but not the transferable features. The next question is what features joint training relies on to fit the target training labels. Figure~\ref{fig:tsne} shows strong correlation between joint training model's features and labels on the target training set, but much less correlation on the target test set, suggesting that the joint training model's feature extractor, applied on the target data (which doesn't have source-specific features), overfits to the target training set. This corroborates the poor accuracy of joint training on the target test set (Figure~\ref{fig:semi}), which is similar to target-only's.\footnote{As sanity checks, when the source contains only transferable features (Figure~\ref{fig:semi}, right, source = A), fine-tuning works well, and when no transferable features  (Figure~\ref{fig:semi}, right, source = B), it does not. }

In Section~\ref{sec:theory}, we rigorously analyze the behavior of these algorithms on a much simplified settings and show that the phenomena above can theoretically occur.  
\section{\merlin: Meta Representation Learning}\label{sec:method}

In this section, we design a meta representation learning algorithm that encourages the discovery of transferable features. As shown in the semi-synthetic experiments, fine-tuning does not have any incentive to learn transferable features if they are not the most convenient for predicting the source labels because it is oblivious to target data. Thus we have to use the source and target together to learn transferable representations. A natural attempt would have been joint training, but it overfits to the target when the target data is scarce as shown in the t-SNE visualizations in Figure~\ref{fig:tsne}. 

To fix the drawbacks of joint training, we recall that good representations should not only work well for the target \textit{training} set but  also \textit{generalize} to the target distribution.  More concretely, a good representation $h_\phi$ should enable the generalization of the linear head learned on top of it---a linear head $\theta$ that is learned by fixing the feature $h_\phi(x)$ as the inputs should generalize well to a held-out dataset. To this end, we design a bi-level optimization objective to learn such features, inspired by meta-learning for fast adaptation~\citep{pmlr-v70-finn17a}  and learning-to-learn for automatic hyperparameter optimization~\citep{maclaurin2015gradient,thrun2012learning} (more discussions below).

We first split the target training set $\DTe$ randomly into $\DTa$ and $\DTb$. Given a feature extractor $\phi$, let $\widehat{\theta}_t(\phi)$ be the linear classifier learned by using features $h_\phi(x)$ as the inputs on the dataset $\DTa$. 
\begin{align}\label{eqn:inner} 
\thetahatphi = \mathop{\arg\min}_{\theta}L_{\DTa}(\theta,\phi)
\end{align}
Note that $\widehat{\theta}_t(\phi)$ depends on the choice of $\phi$ (and is almost uniquely decided by it because the objective is convex in $\theta$). As alluded before, our final objective involves the generalizability  of  $\thetahatphi$ to the held-out dataset $\DTb$:
\begin{align}
\losstarget(\phi) = L_{\DTb}({\thetahatphi} ,\phi) =  \E_{(x,y)\in \DTb} \ell(g_{\thetahatphi}(h_\phi(x)),y).
\end{align}
The final objective is a linear combination of $\losstarget(\phi)$ with the source loss 
\begin{align}
\mathop{\textup{minimize}}_{\phi\in\boldsymbol{\Phi},\theta_s\in\boldsymbol{\Theta}}~ \lossmeta(\phi, \theta_s) :=
L_{\DSe}(\theta_s,\phi)+ \rho \cdot \losstarget(\phi).
\end{align}
To optimize the objective, we can use standard bi-level optimization technique as in learning-to-learn approaches as summarized in Algorithm~\ref{alg:merlin}. We also design a sped-up version of \merlin~to by changing the loss to squared loss so that the $\thetahatphi$ has an analytical solution. More details are provided in Section \ref{sec:speedup} (Algorithm~\ref{alg:merlin_speed}).

\paragraph{Comparison to other meta-learning work.} The key distinction of our approach from MAML \citep{pmlr-v70-finn17a} and other meta-learning algorithms (e.g.,~\citep{nichol2018firstorder, bertinetto2018metalearning}) is that we only have  a single source task and a single target task. Recent work~\citep{Raghu2020Rapid}  argues that feature reuse is the dominating factor of the effectiveness of MAML. In our case, the training target task is exactly the same as the test task, and thus the only possible contributing factor is a better-learned representation instead of fast adaptation. Our algorithm is in fact closer to the work on hyperparameter optimization~\citep{maclaurin2015gradient, zoph2016neural}---if we view the parameters of the head $\theta_t$ as $m$ \textbf{hyperparameters} and view $\phi$ and $\theta_s$ as the ordinary parameters, then our algorithm is tuning hyperparameters on the validation set using gradient descent. 

\begin{algorithm}[t]
\caption{Meta Representation Learning (\merlin).}
\label{alg:merlin}
\begin{algorithmic}[1]
\STATE {\bfseries Input:} the source dataset $\DSe$ and the evolving target dataset $\DTe$.    
\STATE {\bfseries Output:} learned representations $\W$.
\FOR{$i=0$ {\bfseries to} \texttt{MaxIter}} 
  \STATE Initialize the target head $\theta_t^{[0]}$.
  \STATE Randomly sample target train set $\DTa$ and target validation set $\DTb$ from $\DTe$.
  	\FOR{$k=0$ {\bfseries to} $n-1$}
	\STATE Train the target head on $\DTa$:
	\vspace{-5pt}
	\begin{align*}
	\theta_{t}^{[k+1]} \leftarrow \theta_{t}^{[k]} - \eta\nabla_{\theta_{t}^{[k]}}L_{\DTa}(\theta_{t}^{[k]},\phi^{[i]}).
	\end{align*}
    \ENDFOR
\STATE In the outer loop, update the representation $\W$ and the source head $\as$:
\vspace{-5pt}
\begin{align*}
(\W^{[i+1]},\theta_s^{[i+1]}) \leftarrow (\W^{[i]},\theta_s^{[i]}) - \eta\nabla_{(\W^{[i]},\theta_s^{[i]}}\left[ L_{\DSe}(\theta_s^{[i]},\W^{[i]})+\rho L_{\DTb}(\theta_t^{[n]},\phi^{[i]})
\right].
\end{align*}
\vspace{-15pt}
\ENDFOR
\end{algorithmic}
\end{algorithm}

\subsection{MeRLin Learns Transferable Features on Semi-Synthetic Dataset}

We verify that {\merlin} learns transferable features in the semi-synthetic setting of Section~\ref{sec:limitations} where fine-tuning and joint training fail. Figure \ref{fig:semi} (right) shows that \merlin~outperforms fine-tuning and joint training by a large margin and is close to fine-tuning from the source A, which can be almost viewed as an upper bound of any algorithm's performance with AB as the source. Figure \ref{fig:semi2} shows that {\merlin} (trained with source = AB) performs well on A, indicating it learns the transferable features. Figure~\ref{fig:tsne} (\merlin, train\& test) further corroborates the conclusion with the better representations learned by \merlin.  

\section{Theoretical Analysis with Two-layer Quadratic Neural Nets}\label{sec:theory}

The experiments in Section~\ref{sec:limitations} demonstrate the weakness of fine-tuning and joint training. 
On the other hand, {\merlin}  is able to learn the transferable features from the source datasets. In this section, we instantiate transfer learning in a quadratic neural network where the algorithms can be rigorously studied. For a specific data distribution, we prove that (1) fine-tuning and joint training fail to learn transferable features, and (2) MeRLin recovers target ground truth with limited target examples. 

\paragraph{Models.} Consider a two-layer neural network $f_{\agen, \W}(x) = g_{\theta}(h_{\phi}(x))$ with $g_\theta(z) = \theta^\top z$ and $h_\phi = \sigma(\phi^\top x)$, where $\W = \lbbr\w{1}, \w{2}, \cdots, \w{\df}\rbbr \in \Real^{d\times m}$ is the weight of the first layer, $\agen\in\Real^\df$ is the linear head, and $\sigma(\cdot)$ is element-wise quadratic activation. We consider squared loss $\ell(f_{\agen, \W}(x),y) = (f_{\agen, \W}(x)  - y)^2$.  

\paragraph{Source distribution.} Let $k\in \mathbb{Z}^+$ such that $2\le k\le d$.
We consider the following source distribution which can be solved by \textit{multiple possible feature extractors}. Let $x_{[i]}$ denotes the $i$-th entry of $x\in \Real^d$. 
Let $y=0$ happens with prob. $\nicefrac{1}{3}$, and conditioned on $y=0$, we have $x_{[i]}=0$ for $i\le k$,  and $x_{[i]} \sim \{\pm 1, 0\}$ uniformly randomly and independently for $i>k$. With prob.  $\nicefrac{2}{3}$ we have $y=1$, and conditioned on $y=1$, we have $x_{[i]}\sim\{\pm 1\}$ uniformly randomly and independently for $i\le k$, and $x_{[i]} \sim\{\pm 1, 0\}$ uniformly randomly and independently for $i>k$. 

The design choice here is that $x_{[1]},\dots, x_{[k]}$ are the useful entries for predicting the source label, because $y=x_{[i]}^2$ for any $i\le k$. In other words, features $\sigma(e_{[i]}^\top x)$ for $i\le k$ are useful features to learn for the source domain, and any linear mixture of them works. All other entries of $x$ are independent with the label $y$. 

\paragraph{Target distribution.} The target distribution is exactly $k=1$ version of the source distribution. Therefore, $y=x_{[1]}^2$, and $\sigma(e_{[1]}^\top x)$ is the correct feature extractor for the target. All other $x_{[i]}$ for $i > 1$ are independent with the label. 


\paragraph{Source-specific features and transferable features.}
As mentioned before, $\sigma(e_{[1]}^\top x), \cdots, \sigma(e_{[k]}^\top x)$ are all good features for the source, whereas only $\sigma(e_{[1]}^\top x)$ is transferable to the target. 

Since usually the source dataset is much larger than the target dataset, we assume access to \textit{infinite} source data for simplicity, so $\DSe = \DS$. We assume access to $\nt$ target data $\DTe$. 

\paragraph{Regularization:} Because the limited target data, the optimal solutions with unregularized objective are often not unique. Therefore, we study $\ell_2$ regularized version of the baselines and {\merlin}, but we compare them with their own best regularization strength.  Let $\lambda>0$ be the regularization strength. 
The regularized \merlin~objective is $\Lmeta{\lambda}{\as, \W} :=  \lossmeta(\phi, \theta_s)+ \lambda (\norm{\as}^2 + \norm{\W}_F^2)$.
The regularized joint training objective is $\Ljoint{\lambda}{\as,\at,\W}
:= \Ljoint{}{\as,\at,\W}+ \lambda (\norm{\as}^2 + \norm{\at}^2 + \norm{\W}_F^2)\nonumber.$
We also regularize the two objectives in the pre-training and fine-tuning. We pre-train with $\Lsource{\lambda}{\as, \W} := \LossD{ \DS}{\as, \W}  + \lambda (\norm{\as}^2 + \norm{\W}_F^2)$, 
and then only fine-tune the head\footnote{For theoretical analysis we consider only fine-tuning $\at$. It is worth noting that fine-tuning both $\at$ and $\W$ converges to the same solution as target-only training in this setting, which also has large generalization gap due to overfitting.} by minimizing the target loss $\Ltarget{\lambda, \phipre}{\at} := \LossD{\DTe}{\at, \phipre} + \lambda \norm{\at}^2$.

The following theorem shows that neither joint training nor fine-tuning is capable of recovering the target ground truth given limited number of target data.
\begin{theorem}\label{theorem_joint_finetuning}
There exists universal constants $c\in(0,1)$ and $\epsilon>0$, such that so long as $n_t\le cd$, for any $\lambda > 0$, the following statements are true:
	\begin{itemize}[leftmargin=*]
		\setlength{\itemsep}{5pt}
		\setlength{\parskip}{-3pt}
		\item{
With prob. at least $1-4\exp(-\Omega(d))$, the solution $(\hat{\theta}_s,\hat{\theta}_t,\phijoint)$ of the joint training satisfies
			\begin{align}
			\LossD{ \DT}{\hat{\theta}_t, \phijoint}\ge \epsilon.
			\end{align}	}
		\item{
With prob. at least $1-\frac{1}{k}$ (over the randomness of pre-training), the solution $(\hat{\theta}_t, \phipre)$ of the head-only fine-tuning satisfies
			\begin{align}
			\LossD{\DT}{\hat{\theta}_t, \phipre}\ge \epsilon.
			\end{align}	}
	\end{itemize}
\end{theorem}
As will be shown in the proof, not surprisingly, fine-tuning fails because it learns a random feature $\sigma(e_{[i]}^\top x)$ (where $i\in[k]$) for the source during pre-training which does not transfer to the target when $i\ne 1$. Pre-training has no incentive to choose the transferable feature as expected. Joint training fails because it uses one neuron to learn a feature overfitting the target $n_t$ training data exactly, and then use another neuron to learn another feature $\sigma(e_{[i]}^\top x)$ (where $i\in[k]$) to fit the source. In consequence, joint training behaves like training on the source domain and the target domain separately. Training on the source domain does not help learning the target well. The proof of Theorem~\ref{theorem_joint_finetuning} is deferred to Section~\ref{sec:proof}. 

In contrast, the following theorem shows that \merlin~can recover the ground truth of the target task:
\begin{theorem}\label{theorem_meta_transfer}
For any $\lambda < \lambda_0$ where $\lambda_0$ is some universal constant and any failure rate $\xi>0$, if the target set size $\nt >\Theta(\log\frac{k}{\xi})$, with probability at least $1-\xi$, the feature extractor $\phimeta$ found by {\merlin} and the head $\thetahatphimeta$ trained on $\DTa$ recovers the ground truth of the target task: 
	\begin{align}
	\LossD{\DT}{\thetahatphimeta, \phimeta}= 0.
	\end{align}
\end{theorem}

Intuitively, \merlin~learns the transferable feature $\sigma(e_{[1]}^\top x)$  because its simultaneously fits the source and enables the generalization of the head on the target.
The proof can be found in Section~\ref{sec:proof}.
\section{Experiments}

We evaluate \merlin~on several vision and NLP datasets. We show that (1) \merlin~consistently improves over baseline transfer learning algorithms including fine-tuning and joint training in both vision and NLP (Section~\ref{sec:results}), and (2) as indicated by our theory, {\merlin} succeeds because it learns features that are more transferable than fine-tuning and joint training (Section~\ref{sec:analysis}).

\subsection{Setup: Tasks, Models,  Baselines, and Our Algorithms}

The evaluation metric for all tasks is the top-1 accuracy. We run all tasks for 3 times and report their means and standard deviations. Further experimental details are deferred to Section~\ref{sec:details}. 

\subsubsection{Datasets and models}

We consider the following four settings. The first three are object recognition problems (with different label sets). The fourth problem is the prominent NLP benchmark where the source is a language modeling task and the targets are classification problems.

\begin{table}
\centering
  \addtolength{\tabcolsep}{-4.2pt}
\caption{\textbf{Accuracy $(\%$)} on computer vision tasks.}\label{table:acc0}
  \begin{tabular}{l|c|c|c|c|c|c}
  \toprule
  Source & Fashion & SVHN & \multicolumn{3}{c|}{ImageNet} & Food-101\\
  \midrule
  Target & \multicolumn{2}{c|}{USPS (600)} & CUB-200 & Caltech-256 & Stanford Cars & CUB-200\\
  \midrule
  Target-only & 91.07 $\pm$ 0.45 & 91.07 $\pm$ 0.45 & 32.05 $\pm$ 0.67 & 45.63 $\pm$ 1.26 & 23.22 $\pm$ 1.02 & 32.13 $\pm$ 0.64\\
  Joint training & 89.59 $\pm$ 0.56 & 91.54 $\pm$ 0.32 & 55.81 $\pm$ 1.36 & 78.20 $\pm$ 0.50 & 63.25 $\pm$ 0.72 & 42.08 $\pm$ 0.59\\
  Fine-tuning & 90.80 $\pm$ 0.20 & 92.12 $\pm$ 0.39 & 72.52 $\pm$ 0.51 & 81.12 $\pm$ 0.27& 81.59 $\pm$ 0.49 & 52.30 $\pm$ 0.51\\
  L2-sp  & 89.74 $\pm$ 0.41 & 91.86 $\pm$ 0.27 & 73.20 $\pm$ 0.38 & 82.31 $\pm$ 0.22 & 81.26 $\pm$ 0.27 & 53.84 $\pm$ 0.37\\
  \midrule
  \textbf{\merlin} & \textbf{93.34} $\pm$ 0.41 & \textbf{93.10} $\pm$ 0.38 & \textbf{75.42} $\pm$ 0.47 & \textbf{82.45} $\pm$ 0.26 & \textbf{83.68} $\pm$ 0.57 & \textbf{58.68} $\pm$ 0.43 \\
  \bottomrule
  \end{tabular}
\end{table}

\begin{table}[t]
\centering
  \addtolength{\tabcolsep}{20pt}
  \caption{\textbf{Accuracy $(\%$)} of BERT-base on GLUE sub-tasks \texttt{dev} set.}\label{table:acc2}
  \begin{tabular}{l|c|c|c}
  \toprule
  Target & MRPC & RTE & QNLI\\
  \midrule
  Fine-tuning & 83.74 $\pm$ 0.93 & 68.35 $\pm$ 0.86 & 91.54 $\pm$ 0.25\\
  L2-sp& 84.31 $\pm$ 0.37 & 67.50 $\pm$ 0.62 & 91.29 $\pm$ 0.36  \\
  \midrule
  \textbf{\merlinfinetune} & \textbf{86.03} $\pm$ 0.25 & \textbf{70.22} $\pm$ 0.86 & \textbf{92.10} $\pm$ 0.27\\
  \bottomrule
  \end{tabular}
\end{table}

\noindent\textbf{SVHN or Fashion-MNIST $\rightarrow$ USPS.}  
We use either SVHN~\citep{cite:NIPS11SVHN} (73K street view house numbers) or Fashion-MNIST \cite{DBLP:journals/corr/abs-1708-07747} (50K clothes) as the source dataset. The target dataset is a random subset of 600 examples of USPS~\citep{cite:TPAMI94USPS}, a hand-written digit dataset. We down-sampled USPS to simulate the setting where the target dataset is much smaller than the source. We use \textit{LeNet} \citep{cite:IEEE98MNIST}, a three-layer ReLU network in this experiment. 

\noindent\textbf{ImageNet $\rightarrow$ CUB-200, Stanford Cars, or Caltech-256.} To validate our method on real-world vision tasks, we use ImageNet \citep{cite:ILSVRC15} as the source dataset. The target dataset is Caltech-256~\citep{cite:TR07Caltech}, CUB-200 \citep{WahCUB_200_2011}, or Stanford Cars \citep{KrauseStarkDengFei-Fei_3DRR2013}. These datasets have 25468, 5994, 8144 labeled examples respectively, much smaller than ImageNet with 1.2M labeled examples. Caltech is a general image classification dataset of 256 classes. Stanford Cars and CUB are fine-grained classification datasets with 196 categories of cars and 200 categories of birds, respectively. 
We use ResNet-18 \citep{cite:CVPR16DRL}. 

\noindent\textbf{Food-101 $\rightarrow$ CUB-200.} Food \citep{bossard14} is a fine-grained classification dataset of 101 classes of food. Here we validate \merlin~when the gap between the source and target is large.

\noindent\textbf{Language modeling \raw~ GLUE.}  Pre-training on language modeling tasks with gigantic text datasets and fine-tuning on labeled dataset such as GLUE~\citep{wang2018glue} is dominant following the success of BERT~\citep{devlin-etal-2019-bert}. We fine-tune BERT with \merlin~and evaluate it on the three tasks of GLUE with the smallest number of labeled examples,
which standard fine-tuning likely overfits.  

\subsubsection{Baselines} (1) \textit{\targetonly}, (2) \textit{fine-tuning}, and (3) \textit{joint-training} have been defined in Section~\ref{sec:preliminaries}. Following standard practice, the initial learning rate of fine-tuning is $0.1\times$ the initial learning rate of pre-training to avoid overfitting. For join training, the overall objective can be formulated as: $(1-\alpha)L_{\DSe}+ \alpha L_{\DTe}$. We tune $\alpha$ to achieve optimal performance. The fourth baseline is (4) \textit{L2-sp} \citep{pmlr-v80-li18a}, which fine-tunes the models with a regularization penalizing the parameter distance to the pre-trained feature extractor. We also tuned the strength the L2-sp regularizer.

\subsubsection{Our method}
\textbf{\merlin.} We perform standard training with cross entropy loss on the source domain while meta-learning the  representation in the target domain as described in Section~\ref{sec:method}. 

\noindent\textbf{\merlinfinetune.} In BERT experiments, training on the source masked language modeling task is prohibitively time-consuming, so we opt to a light-weight variant instead: start from pre-trained BERT, and only meta-learn the representation in the target domain. 

\subsection{Results}\label{sec:results}

Results of digits classification and object recognition are provided in Table \ref{table:acc0}. \merlin~consistently outperforms all baselines. Note that the discrepancy between Fashion-MNIST and USPS is very large, where fine-tuning and joint training perform even worse than target-only. Nonetheless, \merlin~is still capable of harnessing the knowledge from the source domain. On Food-101$\rightarrow$CUB-200, \merlin~improves over fine-tuning by $6.58\%$, indicating that \merlin~helps learn transferable features even when the gap between the source and target tasks is huge.
In Table \ref{table:acc2}, we validate the our method on GLUE tasks. \merlin-ft outperforms standard BERT fine-tuning and L2-sp. Since \merlin-ft only changes the training objective of fine-tuning, it can be easily applied to NLP models.

\begin{wrapfigure}{r}{5.4cm}\centering
\vspace{-40pt}
\includegraphics[width=0.33\textwidth]{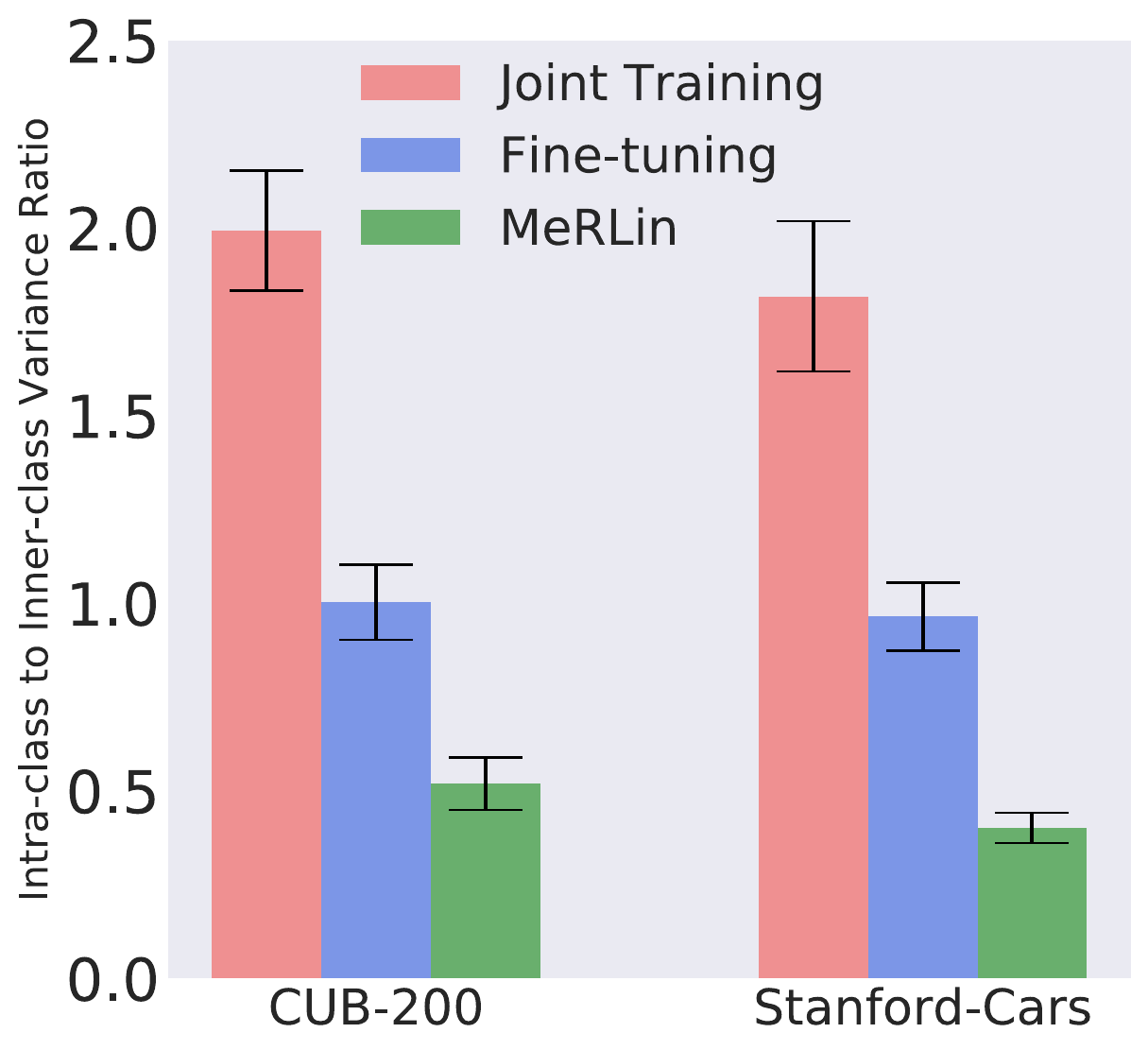}
\vspace{-25pt}
\caption{\textbf{Intra-class to inter-class variance ratio.} This quantity is lowest for~\merlin, indicating that it separates classes best.}
\vspace{-35pt}
\label{fig:ratio}
\end{wrapfigure}

\subsection{Analysis}\label{sec:analysis}

We empirically analyze the representations and verify that \merlin~indeed learns more transferable features than fine-tuning and joint training.

\noindent\textbf{Intra-class to inter-class variance ratio.} Suppose the representation of the j-th example of the i-th class is $\phi_{i,j}$. $\mu_i = \frac{1}{N_i}\sum_{j=1}^{N_i}\phi_{i,j}$, and $\mu = \frac{1}{C}\sum_{i=1}^{C}\mu_i$. Then the intra-class to inter-class variance ratio can be calculated as
$\frac{\sigma^2_{intra}}{\sigma^2_{inner}}=\frac{C}{N}\frac{\sum_{i,j}\|\phi_{i,j}-\mu_i\|^2}{\sum_{i}\|\mu_{i}-\mu\|^2}$.
Low values of this ratio correspond to representations where classes are well-separated. Results on ImageNet $\rightarrow$ CUB-200 and Stanford Cars task are shown in Figure \ref{fig:ratio}. \merlin~reaches much smaller ratio than baselines.
\section{Additional Related Work}

Transfer learning has become one of the underlying factors contributing to the success of deep learning applications. In computer vision, ImageNet pre-training is a common practice for nearly all target tasks. Early works \citep{Oquab_2014_CVPR,pmlr-v32-donahue14} directly apply ImageNet features to downstream tasks. Fine-tuning from ImageNet pre-trained models have become dominant ever since \citep{cite:NIPS14CNN,7298965,Girshick_2014_CVPR,He_2017_ICCV,alex2019big}. On the other hand, transfer learning is also crucial to the success of NLP algorithms. Pre-trained transformers on large-scale language tasks boosts performance on downstream tasks. \citep{NIPS2017_7181,devlin-etal-2019-bert,NIPS2019_8812}.

A recent line of literature casts doubt on the consistency of transfer learning's success \citep{DBLP:journals/corr/HuhAE16,2018ImageNet,NIPS2019_8596,cite:arxivrethink,Kornblith_2019_CVPR,liu2019understanding,neyshabur2020transferred}. \citet{DBLP:journals/corr/HuhAE16} observed that some set of examples in ImageNet are more transferable than the others. \citet{2018ImageNet} found out that the texture of ImageNet is not transferable to some target tasks. Training on the source dataset may also need early stopping to find optimal transferability \citet{liu2019understanding,neyshabur2020transferred}.

Meta-learning, originated from the \textit{learning to learn} idea~\citep{2001Learning,2002Vilalta, maclaurin2015gradient, zoph2016neural}, learns from multiple training tasks models  that can be swiftly adapted to new tasks~\citep{pmlr-v70-finn17a, NIPS2019_8306,nichol2018firstorder}. \citet{Raghu2020Rapid,goldblum2020unraveling} empirically studied the mechanism of  MAML's success. Computationally, our method uses bi-level optimization techniques similar to meta-learning work. E.g., \citet{bertinetto2018metalearning} speeds up the implementation of MAML~\cite{pmlr-v70-finn17a} with closed-form solution of the inner loop, which is a technique that we also use. However, the key difference between our paper from the meta-learning approach is that we only learn from a single target task and evaluate on it. Therefore, conceptually, our algorithm is  closer to the learning-to-learn approach  for hyperparameter optimization~\citep{maclaurin2015gradient, zoph2016neural}, where there is a single distribution that generates the  training and validation dataset. \citet{Raghu2020Rapid,goldblum2020unraveling} empirically studied the success of MAML. \citet{pmlr-v97-balcan19a,tripuraneni2020provable} theoretically studied meta-learning in a few-shot learning setting. 

\section{Conclusion}

We study the limitations of fine-tuning and joint training. To overcome their drawbacks, we propose meta representation learning to learn transferable features. Both theoretical and empirical evidence verify our findings. Results on vision and NLP tasks validate our method on real-world datasets. Our work raises many intriguing questions for further study. Could we apply meta-learning to heterogeneous target tasks? What's more, future work can pay attention to disentangling transferable features from non-transferable features explicitly for better transfer learning. 

\section*{Acknowledgement}
HL thanks Mingsheng Long for discussions of experiments. CW acknowledges support from an NSF Graduate Research Fellowship. TM acknowledges support of Google Faculty Award. 

\bibliography{main}
\bibliographystyle{plainnat}

\newpage
\appendix

\section{Additional Details of Experiments}\label{sec:details}
\subsection{The Semi-synthetic Experiment}\label{sec:semi_details}
The original CIFAR images is of resolution $32\times32$. For the transferable dataset A, we reserve the upper $16\times32$ and fill the lower half with $[0.485, 0.456, 0.406]$ for the three channels (the mean of CIFAR-10 images). For the non-transferable dataset B, the lower part $16\times32$ pixels are generated with i.i.d. gaussian distribution with the upper half filled with $[0.485, 0.456, 0.406]$ similarly. To make the non-transferable part related to the labels, we set the mean of the gaussian distribution to $0.1\times c$, where $c$ is the class index of the image. The variance of the gaussian noise is set to $0.2$. We always clamp the images to $[0,1]$ to make the generated images valid. For the source dataset, we use $49500$ CIFAR-10 images, while for the target, we use the other $500$ to avoid memorizing target examples. 

We use ResNet-32 implementation provided in \url{github.com/akamaster/pytorch_resnet_cifar10}. We set the initial learning rate to $0.1$, and decay the learning rate by $0.1$ after every 50 epochs.  We use t-SNE \citep{Hinton2008Visualizing} visualizations provided in \texttt{sklearn}. The perplexity is set to 80.

\subsection{Implementation on Real Datasets}
We implement all models on PyTorch with 2080Ti GPUs. All models are optimized by SGD with 0.9 momentum. For digit classification tasks, the initial learning rate is set to 0.01, with $5\times10^{-4}$ weight decay. The batch-size is set to 64. We run each model for 150 epochs. For object recognition tasks, ImageNet pre-trained models can be found in \texttt{torchvision}. We use a batch size of 128 on the source dataset and 512 on the target dataset. The initial learning rate is set to 0.1 for training from scratch and 0.01 for ImageNet initialization. We decay the learning rate by 0.1 every 50 epochs until 150 epochs. The weight decay is set to $5\times10^{-4}$. For GLUE tasks, we follow the standard practice of \citet{devlin-etal-2019-bert}. The BERT model is provided in \url{github.com/Meelfy/pytorch_pretrained_BERT}.
For each model, we set the head (classifier) to the top one linear layer.
We use a batch size of 32. The learning rate is set to $5\times 10^{-5}$ with 0.1 warmup proportion. During fine-tuning, the initial learning rate is 10 times smaller than training from scratch following standard practice. The hyper-parameter $\rho$ is set to $2$, and $\lambda$ is found with cross validation. We also provide the results of varying $\rho$ and $\lambda$ in Section \ref{sec:sensitivity}. 

\subsection{Implementing the Speed-up Version}\label{sec:speedup}
\textbf{Practical implementation: speeding up with MSE loss.} 
Training the head $g_\theta$ in the inner loop of meta learning can be time-consuming. Even using implicit function theorem or implicit gradients as proposed in \citep{NIPS2019_8358, NIPS2019_8306,lin2020modelbased}, we have to approximate the inverse of Hessian. To solve the optimization issues, we propose to analytically calculate the prediction of the linear head ${\theta}_t$ and directly back-prop to the feature extractor $h_\phi$. Thus, we only need to compute the gradient once in a single step. Concretely, suppose we use MSE-loss. Denote by $\mathbf{H}\in\mathbb{R}^{\frac{n_t}{2}\times m}$ the feature matrix of the $\frac{n_t}{2}$ target samples in the target meta-training set $\DTa$. Then $\widehat{\theta}_t$ in \eqref{eqn:inner} can be analytically computed as $\widehat{\theta}_t=(\mathbf{H}\mathbf{H}^\top+\lambda\mathbf{I})^{-1}\mathbf{y}$, where $\lambda$ is a hyper-parameter for regularization. The objective of the outer loop can be directly computed as
\begin{equation}\label{eqn:speedup}
\mathop{\textup{minimize}}_{\phi\in\boldsymbol{\Phi},\theta_s\in\boldsymbol{\Theta}}~ J(\phi, \theta_s) =
L_{\DSe}(\theta_s,\phi)+\rho\frac{2}{n_t}\sum_{i=1}^{\frac{n_t}{2}}\ell(g_{(\mathbf{H}\mathbf{H}^\top+\lambda\mathbf{I})^{-1}\mathbf{y}}\circ h_{\phi}(x_i^{t'}),y_i^{t'}).
\end{equation}

We implement the speed-up version on classification tasks following \citet{NIPS2019_9025}. We treat the classification problems as multi-variate ridge regression. Suppose we have label $c\in\{1,2,\cdots,10\}$. Then the target encoding for regression is $-0.1\times\mathbf{1}+\mathbf{e}_c$. For example, if the label is $3$, then the encoding will be $(-0.1,-0.1,0.9,\cdots,-0.1)$. Then the parameters of the target head in the inner loop can be computed as $\widehat{\theta}_t=(\mathbf{H}\mathbf{H}^\top+\lambda\mathbf{I})^{-1}\mathbf{Y}$. We then compute the MSE loss on the target validation set: $\|\mathbf{H}_{val}\widehat{\theta}_t-\mathbf{Y}_{val}\|_2^2$ in the outer loop. We summarize the details of the vanilla version and the speed-up version in Algorithm~\ref{alg:merlin} and Algorithm~\ref{alg:merlin_speed}.

\begin{algorithm}[h]
\caption{Meta Representation Learning (\merlin): speed-up implementation.}
\label{alg:merlin_speed}
\begin{algorithmic}[1]
\STATE {\bfseries Input:} the source dataset $\DSe$ and the evolving target dataset $\DTe$.    
\STATE {\bfseries Output:} learned representations $\W$.
\FOR{$i=0$ {\bfseries to} \texttt{MaxIter}} 
  \STATE Randomly sample target train set $\DTa$ and target validation set $\DTb$ from $\DTe$.
    \STATE Analytically calculate the solution of target head $\thetahatphi$ in the inner loop
    \vspace{-5pt}
    \begin{align*}
    \widehat{\theta}(\phi^{[i]})=(\mathbf{H}^{[i]}\mathbf{H}^{[i]\top}+\lambda\mathbf{I})^{-1}\mathbf{Y}.
    \end{align*} 
\vspace{-15pt}
\STATE In the outer loop, update the representation $\W$ and the source head $\as$:
\vspace{-5pt}
\begin{align*}
(\W^{[i+1]},\theta_s^{[i+1]}) \leftarrow (\W^{[i+1]},\theta_s^{[i+1]}) - \eta\nabla_{(\W^{[i]},\theta_s^{[i]})}\left[ L_{\DSe}(\theta_s^{[i]},\phi^{[i]})+\rho\frac{2}{n_t}\left\|\mathbf{H}^{[i]}_{val}\widehat{\theta}_t(\phi^{[i]})-\mathbf{Y}_{val}\right\|_2^2
\right].
\end{align*}
\vspace{-15pt}
\ENDFOR
\end{algorithmic}
\end{algorithm}

\subsection{Datasets}
We provide details and links of datasets below.

\textbf{CUB-200} \cite{WahCUB_200_2011} is a fine-grained dataset of 200 bird species. The training dataset consists of 5994 images and the test set consists of 5794 images. \url{http://www.vision.caltech.edu/visipedia/CUB-200-2011.html}

\textbf{Stanford Cars} \cite{KrauseStarkDengFei-Fei_3DRR2013} dataset contains 16,185 images of 196 classes of cars. The data is split into 8,144 training images and 8,041 testing images. \url{http://ai.stanford.edu/~jkrause/cars/car_dataset.html}

\textbf{Food-101} \cite{bossard14} is a fine-grained dataset of 101 kinds of food, with 750 training images and 250 test images for each kind. \url{http://www.vision.ee.ethz.ch/datasets_extra/food-101/}

\textbf{Caltech-256} is a object recognition dataset of 256 categories. In our experiments, the training set consists of 25468 images, and the test set consists of 5139 images.\url{http://www.vision.caltech.edu/Image_Datasets/Caltech256/}

\textbf{MNIST} \cite{cite:IEEE98MNIST} is a dataset of hand-written digits. It has a training set of 60,000 examples, and a test set of 10,000 examples. \url{http://yann.lecun.com/exdb/mnist/}

\textbf{SVHN} \cite{cite:NIPS11SVHN} is a real-world image dataset of street view house numbers. It has 73257 digits for training, 26032 digits for testing.  \url{http://ufldl.stanford.edu/housenumbers/}

\subsection{Further Ablation Study.} 
We extend the last column of Table~\ref{table:acc2} in Table~\ref{table:ablation}. We further compare with two variants of \merlin~as ablation study: 

\textbf{\merlin~(pre-trained).} We first pre-train the model on the source dataset and then optimize the \merlin~objective starting from the pre-trained solution.

\textbf{\merlin-target-only.} \merlin-target-only only meta-learns representations on the target domain starting from random initialization. We test whether the meta-learning objective itself has regularization effect.

\begin{table}[htbp]

\addtolength{\tabcolsep}{-2.5pt}
\caption{Accuracy on Food $\rightarrow$ CUB.}
\label{table:ablation}
\centering
\begin{small}
\begin{tabular}{l|c|c|c|c|c|c}
\toprule
Algorithm & Target-only & Fine-tuning & Joint Training & \merlin-target-only & \merlin~(pre-trained) & \merlin \\
\midrule
Accuracy & 32.10 $\pm$ 0.64 & 52.30 $\pm$ 0.51 & 42.08 $\pm$ 0.59 & 40.17 $\pm$ 0.44 & 55.26 $\pm$ 0.43 &  58.68 $\pm$ 0.43\\
\bottomrule
\end{tabular}
\end{small}
\end{table}

\merlin~(pre-trained) performs worse than \merlin, but it still improves over fine-tuning and joint training. Note that \merlin~(pre-trained) only need to train on ImageNet for 2 epochs, much shorter than joint training. \merlin-target-only improves target-only by $8\%$, indicating that meta-learning helps avoid overfitting even without the source dataset.
\subsection{Feature-label correlation}

\begin{figure}[h]
  \centering
  \subfigure[Comparison of feature-label correlation.]{
   \includegraphics[width=0.55\textwidth]{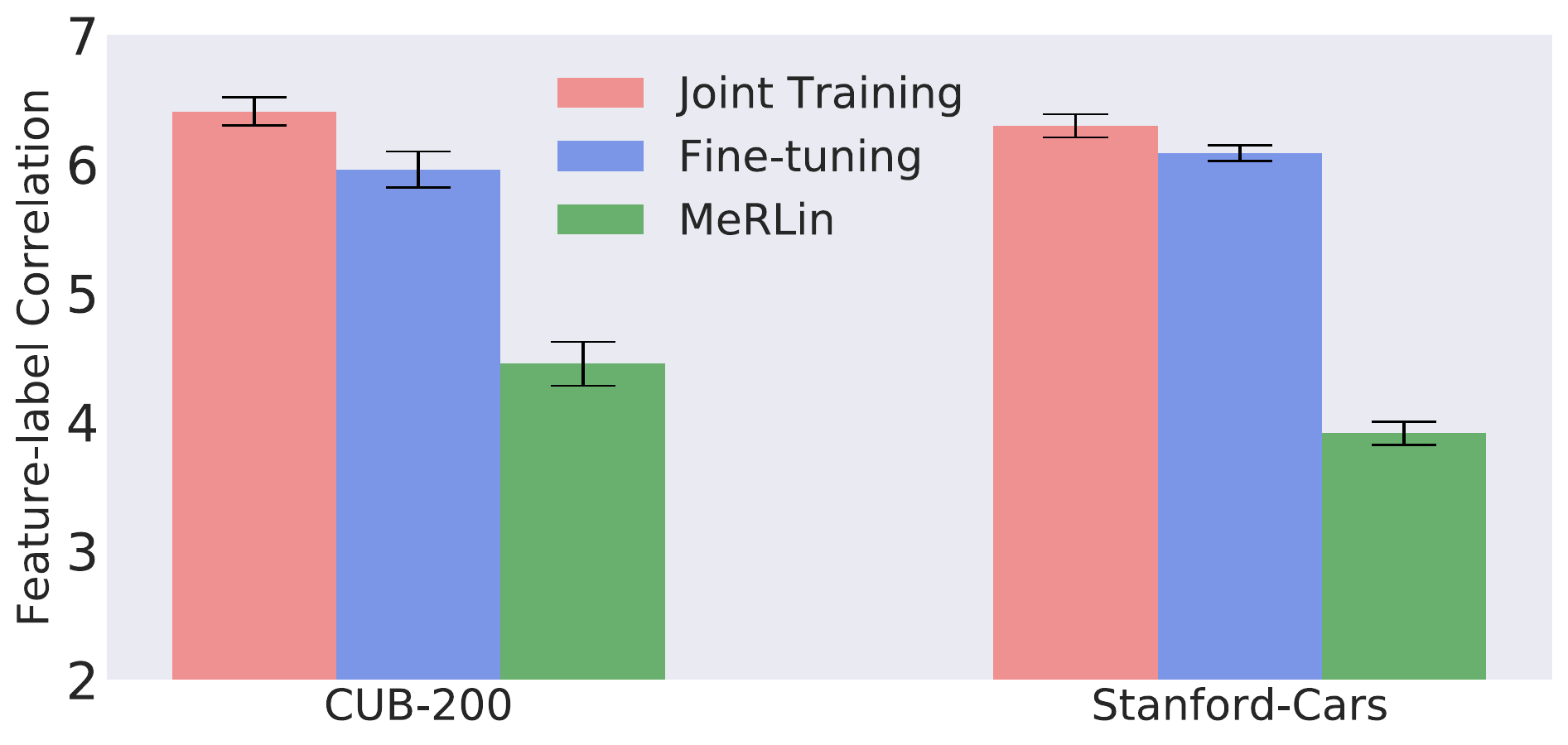}
    \label{fig:correlation}
    }
    \hspace{10pt}
      \subfigure[Sensitivity to hyper-parameters.]{
    \includegraphics[width=0.345\textwidth]{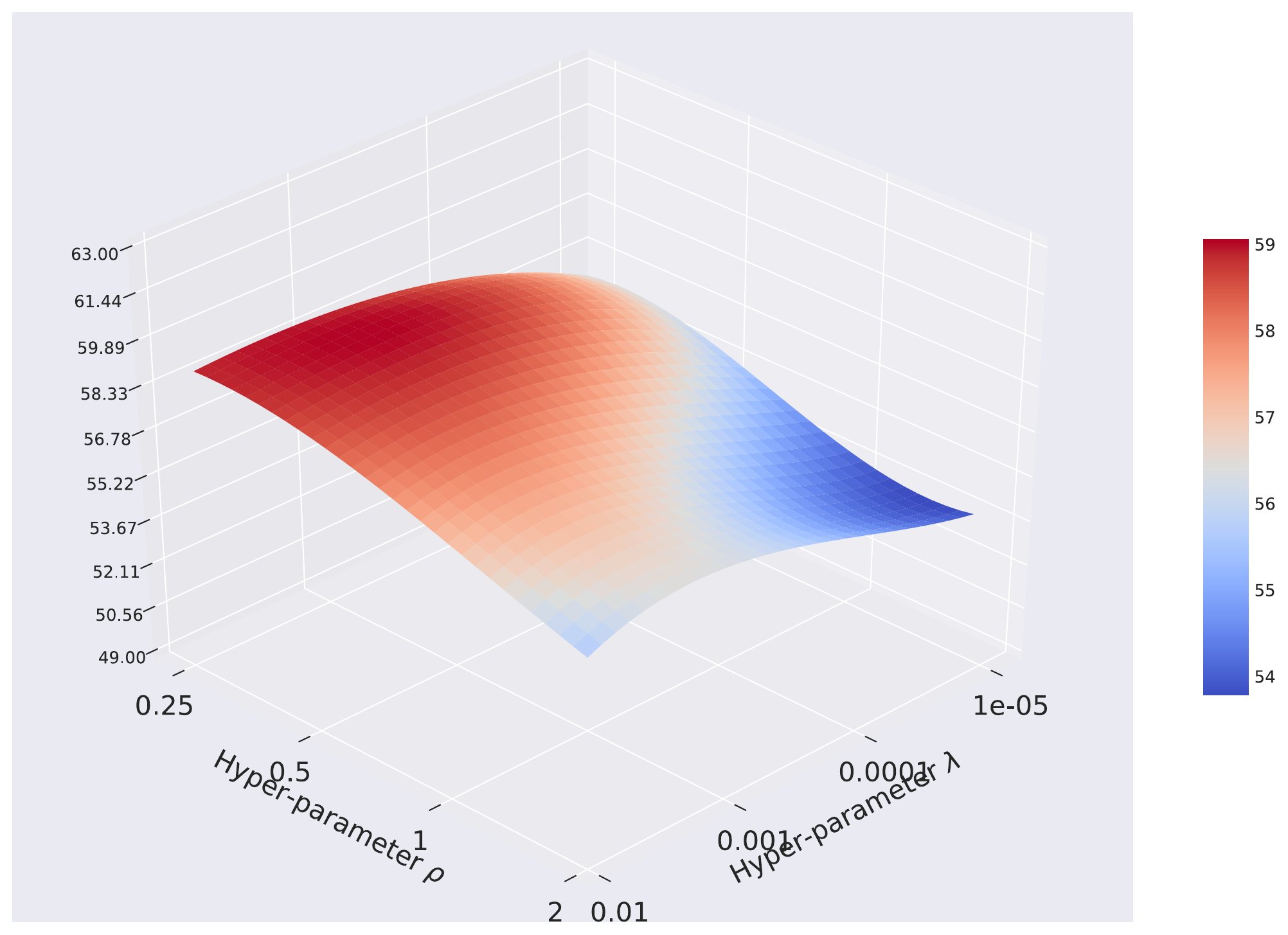}
    \label{fig:hyper}
    }
 \caption{(a) \textbf{Analysis of Feature Quality.} Comparison of feature-label correlation. A lower quantity is better, and~\merlin~has the lowest value. (b) textbf{Sensitivity of the proposed method to hyper-parameters.} We test the accuracy on Food-101$\rightarrow$CUB-200 with varying $\rho$ and $\lambda$ and provide the visualization. }
  \label{fig2}
\end{figure}

Suppose the feature matrix is $\mathbf{H}$, and the label vector is $\mathbf{y}$, then the correlation between feature and label can be defined as $\mathbf{y}^\top\left(\mathbf{H}^\top\mathbf{H}\right)^{-1}\mathbf{y}$. As is shown by \citet{pmlr-v97-arora19a,NIPS2019_9266}, this term is closely related to the generalization error of neural networks, with a smaller quantity indicating better generalization. We calculate $\mathbf{y}^\top\left(\mathbf{H}^\top\mathbf{H}\right)^{-1}\mathbf{y}$ on ImageNet $\rightarrow$ CUB-200 and Stanford Cars. As shown in Figure \ref{fig:correlation}, the features learned by \merlin~are more closely related to labels than fine-tuning and joint training, indicating \merlin~is indeed learning more transferable features compared with baselines.

\subsection{Sensitivity of the Proposed Method to Hyper-parameters.}\label{sec:sensitivity}

We test the model on Food-101$\rightarrow$CUB-200 with varying hyper-parameters $\rho$ and $\lambda$. Results in Figure \ref{fig:hyper} indicate that the model is not sensitive to varying $\rho$ and $\lambda$. Intuitively, larger $\rho$ indicates more emphasis on the target meta-task. When $\rho$ approaches $0$, the performance of MRL is approaching fine-tuning. $\lambda$ exerts regularization to the classifier in the inner loop training. It is also note worthy that $\lambda$ can avoid the problem that $\mathbf{H}\mathbf{H}^\top$ is occasionally invertible. Without $\lambda$ the model can fail to converge sometimes.
\clearpage
\newpage
\section{Missing Details in Section \ref{sec:theory}}\label{sec:proof}

\subsection{Proof of Theorem~\ref{theorem_joint_finetuning}}
\begin{lemma}\label{lemma_norm_solution}
	Suppose $0\le\epsilon\le \frac{2}{3}$. For each solution $\as, \at, \W$ satisfying $\Exp_{x,y\sim \DT}\lbbr\ell(f_{\at, \W}(x),y)\rbbr \leq \epsilon$, the joint training objective $\Ljoint{\lambda}{\as, \at, \W}$ is lower bounded: 
	\begin{align*}
	\Ljoint{\lambda}{\as, \at, \W} &=(1-\alpha) \LossD{\DS}{\as, \W}+ \alpha\LossD{\DTe}{\at, \W} + \lambda \lbr\norm{\as}^2 + \norm{\at}^2 + \norm{\W}_F^2\rbr\\
	&\ge \min_\mu \lbr \frac{3\lambda}{2^{4/3}}|\mu|^{2/3} + \frac{2}{3}(1-\alpha)(\mu-1)^2\rbr + \frac{3\lambda}{2^{4/3}} \lbr 1-\sqrt{\frac{3\epsilon}{2}}\rbr^{2/3}.
	\end{align*}
\end{lemma}
\begin{proof}[Proof of Lemma~\ref{lemma_norm_solution}]

	Define $d\times d$ matrix 
\begin{align}
A = \sum_{i=1}^{\df} \as_i \w{i}\w{i}^\top. 
\end{align}

Define $x_{[1:k]}$ and $x_{[k+1:d]}$ be the first $k$ and last $d-k$ dimensions of $x$. $A_{k,k}$, $A_{k, \bar{k}}$ and $A_{\bar{k}, \bar{k}}$ be $k\times k$, $k\times (d-k)$ and $(d-k)\times (d-k)$ matrices that correspond to the upper left, upper right and lower right part of $A$. For a random vector $x$ where the first $k$ dimensions are uniformly independently from $\{\pm 1\}$, the last $d-k$ dimensions are uniformly indepdently from $\{0, \pm1\}$, define random variables
$A_1 = x_{[1:k]}^\top A_{k,k} x_{[1:k]}$, $A_2 = x_{[k+1:d]}^\top A_{\bar{k},\bar{k}} x_{[k+1:d]}$. (Note that  $x$ is defined on a different distribution than $\DS$.)

We have bound
\begin{align}\label{eq:source_part_lb}
&  (1-\alpha)\Exp_{x,y\sim \DS}\lbbr\ell(f_{\as, \W}(x),y)\rbbr + \lambda \lbr\norm{\as}^2 + \frac{1}{2}\norm{\W}_F^2\rbr \\
= &(1-\alpha)\Exp_{x,y\sim \DS} \lbbr \lbr x_{[1:k]}^\top A_{k,k} x_{[1:k]} + 2x_{[1:k]}^\top A_{k,\bar{k}} x_{[k+1:d]}  +   x_{[k+1:d]}^\top A_{\bar{k},\bar{k}} x_{[k+1:d]} - y\rbr^2 \rbbr\nonumber
\\& \quad\quad\quad+ \lambda \lbr\norm{\as}^2 + \frac{1}{2}\norm{\W}_F^2\rbr\\
\ge & (1-\alpha)\Exp_{x,y\sim \DS} \lbbr \lbr x_{[1:k]}^\top A_{k,k} x_{[1:k]} +  x_{[k+1:d]}^\top A_{\bar{k},\bar{k}} x_{[k+1:d]} - y\rbr^2 \rbbr+\lambda \lbr\norm{\as}^2 + \frac{1}{2}\norm{\W}_F^2\rbr\\
= &(1-\alpha)\lbr \frac{2}{3}\Exp\lbbr \lbr A_1 - 1 \rbr^2 \rbbr + \frac{4}{3}\Exp \lbbr  (A_1-1) A_2\rbbr + \Exp\lbbr A_2^2\rbbr\rbr+ \lambda \lbr\norm{\as}^2 + \frac{1}{2}\norm{\W}_F^2\rbr\\
= & (1-\alpha)\lbr\frac{2}{3} \Exp\lbbr\lbr  A_1 +  A_2  - 1 \rbr^2 \rbbr+ \frac{1}{3} \Exp\lbbr A_2^2\rbbr\rbr+ \lambda \lbr\norm{\as}^2 + \frac{1}{2}\norm{\W}_F^2\rbr\\
\ge & \frac{2}{3}(1-\alpha) \lbr\Exp\lbbr A_1 +  A_2 \rbbr - 1 \rbr^2 + \frac{3\lambda}{2^{4/3}} \lbr|\Exp\lbbr A_1 + A_2 \rbbr|\rbr^{2/3}
\end{align}
The first inequality is because 
\begin{align}
\Exp_{x,y\sim \DS}\lbbr \lbr x_{[1:k]}^\top A_{k,k} x_{[1:k]} \rbr\lbr x_{[1:k]}^\top A_{k,\bar{k}} x_{[k+1:d]} \rbr \rbbr = 0,
\end{align}
\begin{align}
\Exp_{x,y\sim \DS}\lbbr \lbr  x_{[k+1:d]}^\top A_{\bar{k},\bar{k}} x_{[k+1:d]}  \rbr\lbr x_{[1:k]}^\top A_{k,\bar{k}} x_{[k+1:d]} \rbr \rbbr = 0,
\end{align}
\begin{align}
\Exp_{x,y\sim \DS}\lbbr \lbr x_{[1:k]}^\top A_{k,\bar{k}} x_{[k+1:d]} \rbr y  \rbbr = 0.
\end{align}
The second inequality is because
\begin{align}
\sum_{i=1}^{\df} \lbr \lbr\as_i\rbr^2 + \frac{1}{2}\norm{\w{i}}^2 \rbr 
\ge& \frac{3}{2^{4/3}} \sum_{i=1}^{\df} \lbr |\as_i| \cdot \norm{\w{i}}^2\rbr^{2/3}\\
\ge& \frac{3}{2^{4/3}}  \lbr \sum_{i=1}^{\df} |\as_i| \cdot \norm{\w{i}}^2 \rbr^{2/3}\\
\ge& \frac{3}{2^{4/3}}  \lbr \sum_{i=1}^{\df} |A_{[i,i]}| \rbr^{2/3},
\end{align}
where the first inequality is AM-GM inequality, the second inequality is by concavity of $(\cdot)^{2/3}$. The third inequality is because for diagonal matrix $D$ that has $1$ at $(i,i)$ if $A_{[i,i]}\ge0$, $-1$ at $(i,i)$ if $A_{[i,i]}< 0$, we have
\begin{align}
\sum_{i=1}^{\df} |A_{[i,i]}|  = tr(AD) =  \sum_{i=1}^{\df} \as_i \w{i}^\top D \w{i} \le \sum_{i=1}^{\df} |\as_i| \cdot \norm{\w{i}}^2 .
\end{align}

On the other hand, for the target, we define $d\times d$ matrix 
\begin{align}
B = \sum_{i=1}^{\df} \at_i \w{i}\w{i}^\top. 
\end{align}

Define $x_{[1]}$ and $x_{[2:d]}$ be the first $1$ and last $d-1$ dimensions of $x$. $B_{1,1}$, $B_{1, \bar{1}}$ and $B_{\bar{1}, \bar{1}}$ be $1\times 1$, $1\times (d-1)$ and $(d-1)\times (d-1)$ matrices that correspond to the upper left, upper right and lower right part of $B$. For a random vector $x$ where the first dimension is uniformly independently from $\{\pm 1\}$, the last $d-1$ dimensions are uniformly indepdently from $\{0, \pm1\}$, define random variables
$B_1 = x_{[1]}^\top B_{1,1} x_{[1]}$, $B_2 = x_{[2:d]}^\top B_{\bar{1},\bar{1}} x_{[2:d]}$. (Note that  $x$ is defined on a different distribution than $\DT$.)

Using similar argument as above, we have
\begin{align}
\Exp_{x,y\sim \DT}\lbbr\ell(f_{\at, \W}(x),y)\rbbr \ge \frac{2}{3} \lbr\Exp\lbbr B_1 +  B_2 \rbbr - 1 \rbr^2.
\end{align}
However, we know that $\Exp_{x,y\sim \DT}\lbbr\ell(f_{\at, \W}(x),y)\rbbr \leq \epsilon$, so there has to be 
\begin{align}
	\Exp\lbbr B_1 + B_2 \rbbr  \ge 1-\sqrt{\frac{3\epsilon}{2}},
\end{align}
therefore we have
\begin{align}\label{eq:target_part_lb}
\lambda\lbr \norm{\at}^2 + \frac{1}{2}\norm{\W}_F^2 \rbr \ge \frac{3\lambda}{2^{4/3}} \lbr1-\sqrt{\frac{3\epsilon}{2}}\rbr^{2/3}.
\end{align}
Summing up Equation~\ref{eq:source_part_lb} and Equation~\ref{eq:target_part_lb} finishes the proof.

\end{proof}

\begin{lemma}\label{lemma_norm_upperbound}
	Assume $\hat{\phi}_t$ is a vector such that $\langle \hat{\phi}_t, x^t_i \rangle = x^t_{i[1]}$ for all $x^t_i\in \DTe$, then there exists some solution $(\as, \at, \W)$ such that 
	\begin{align}\label{eq:lemma_norm_upperbound_1}
		\Ljoint{\lambda}{\as, \at, \W} \le \min_\mu \lbr \frac{3\lambda}{2^{2/3}}|\mu|^{2/3} + \frac{2}{3}(1-\alpha)(\mu-1)^2\rbr + \frac{3\lambda}{2^{2/3}} \norm{\hat{\phi}_t}^{4/3}_2.
	\end{align}
\end{lemma}
\begin{proof}[Proof of Lemma~\ref{lemma_norm_upperbound}]
	Assume $\mu^* \in \argmin_\mu \lbr \frac{3\lambda}{2^{2/3}}|\mu|^{2/3} + \frac{2}{3}(1-\alpha)(\mu-1)^2\rbr$, then obviously $\mu^* \in [0, 1]$. Let $\phi_1 = (\sqrt{2}\mu^*)^{1/3}e_1$, $\phi_2 = \frac{2^{1/6}}{\norm{\hat{\phi}_t}_2^{1/3}}\hat{\phi}_t$, $\phi_i=0$ for $i>2$, $\as = (\mu^*/2)^{1/3}e_1$ and $\at = \frac{\norm{\hat{\phi}_t}^{2/3}}{2^{1/3}} e_2$. Now we prove that this model satisfies the Equation~\ref{eq:lemma_norm_upperbound_1}.
	
	First of all, we notice that for any $x^t_i \in \DTe$, there is
	\begin{align}
		&{x^t_i}^\top\lbr\sum_{i=1}^{m} \at_i \phi_{i}\phi_{i}^\top\rbr x^t_i\\
		=& {x^t_i}^\top\at_2 \phi_{2}\phi_{2}^\top x^t_i\\
		=& \langle \hat{\phi}_t, x^t_i \rangle^2\\
		=& y^t_i.
	\end{align}
	Therefore we have 
	\begin{align}\label{eq:lemma_norm_upperbound_2}
	\LossD{\DTe}{\at, \W} =0
	\end{align}.
	
	On the other hand, we have
	\begin{align}\label{eq:lemma_norm_upperbound_3}
		&(1-\alpha)\LossD{\DS}{\as, \W} + \lambda\lbr\norm{\as}^2 +\norm{\phi_1}^2\rbr\\
		=&\frac{2}{3}(1-\alpha)(\mu^*-1)^2 + \frac{3\lambda}{2^{2/3}} |\mu^*|^{2/3}\\
		=& \min_\mu \lbr \frac{3\lambda}{2^{2/3}}|\mu|^{2/3} + \frac{2}{3}(1-\alpha)(\mu-1)^2\rbr.
	\end{align}
	Plugging Equation~\ref{eq:lemma_norm_upperbound_2} and Equation~\ref{eq:lemma_norm_upperbound_3} into the formula of $\Ljoint{\lambda}{\as, \at, \W} $ finishes the proof.
\end{proof}

\begin{lemma}\label{lemma_random_projection}
	Let $X\in\Real^{d\times n}$ be a random matrix where each entry is uniformly random and independently sample from $\{0, \pm1\}$, $n<\frac{d}{2}$. 
	Let $P_X e_1$ be the projection of $e_1$ to the column space of $X$.
	Then, there exists absolute constants $c_0>0$ and $C>0$, such that with probability at least $1-4\exp(-Cd)$, there is 
	\begin{align}
		\norm{P_X e_1}_2 \le c_0 \sqrt{\frac{n}{d}}.
	\end{align}
\end{lemma}

\begin{proof}[Proof of Lemma~\ref{lemma_random_projection}]
	Let $s_{min}(X)$ and $s_{max}(X)$ be the minimal and maximal singular values of $X$ respectively. Then we have
	\begin{align}
		\norm{P_X e_1}_2 &= \norm{X(X^\top X)^{-1}X^\top e_1}_2\\
		&\le \norm{X}_{op} \norm{(X^\top X)^{-1}}_{op} \norm{X^\top e_1}_2\\
		&\le s_{max}(X) (s_{min}(X))^{-2} \sqrt{n}.
	\end{align}	

	By Theorem 3.3 in~\cite{rudelson2010non}, there exists constants $c_1, c_2>0$, such that 
	\begin{align}
		P(s_{min}(X) \le c_1(\sqrt{d}-\sqrt{n})) \le 2\exp(-c_2d).
	\end{align}
	By Proposition 2.4 in~\cite{rudelson2010non}, there exists constants $c_3, c_4>0$, such that
	\begin{align}
P(s_{max}(X) \ge c_3(\sqrt{d}+\sqrt{n})) \le 2\exp(-c_4d).
   \end{align}
	Let $C=min\{c_2,c_4\}$, then with probability at least $1-4\exp(Cd)$, there is
	\begin{align}
		& s_{max}(X) (s_{min}(X))^{-2} \sqrt{n}\\
		\le & \frac{c_3}{c_1^2} \frac{\sqrt{d}+\sqrt{n}}{(\sqrt{d} - \sqrt{n})^2} \sqrt{n}\\
		\le &  \frac{(2+\sqrt{2})c_3}{(\sqrt{2}-1)^2 c_1^2}\sqrt{\frac{n}{d}},
	\end{align}
	which completes the proof.

\end{proof}

\begin{proof}[Proof of Theorem~\ref{theorem_joint_finetuning}]
	We prove the joint training part of Theorem~\ref{theorem_joint_finetuning} following this intuition: (1) the total loss of each solution with target loss $\Exp_{x,y\sim \DT}\lbbr\ell(f_{\at, \W}(x),y)\rbbr \leq \epsilon$ is lower bounded as indicated by Lemma~\ref{lemma_norm_solution}, and (2) there exists a solution with loss smaller than the aforementioned lower bound as indicated by Lemma~\ref{lemma_norm_upperbound}.  

	By Lemma~\ref{lemma_norm_solution}, for any $\as, \at, \W$ satisfying $\Exp_{x,y\sim \DT}\lbbr\ell(f_{\at, \W}(x),y)\rbbr \leq \epsilon$, the joint training loss $\Ljoint{\lambda}{\as, \at, \W}$ is lower bounded, 
	\begin{align}\label{eq:joint_proof_lower_bound}
	\Ljoint{\lambda}{\as, \at, \W} \ge \min_\mu \lbr \frac{3\lambda}{2^{4/3}}|\mu|^{2/3} + \frac{2}{3}(1-\alpha)(\mu-1)^2\rbr + \frac{3\lambda}{2^{4/3}} \lbr 1-\sqrt{\frac{3\epsilon}{2}}\rbr^{2/3}.
	\end{align}

	Let $P_X e_1$ be the projection of vector $e_1$ to the subspace spanned by the target data. According to Lemma~\ref{lemma_norm_upperbound}, there exists some solution $(\as, \at, \W)$ such that 
	\begin{align}\label{eq:joint_proof_upper_bound}
	\Ljoint{\lambda}{\as, \at, \W} \le \mu \lbr \frac{3\lambda}{2^{2/3}}|\mu|^{2/3} + \frac{2}{3}(1-\alpha)(\mu-1)^2\rbr + \frac{3\lambda}{2^{2/3}} \norm{\hat{\phi}_t}^{4/3}_2.
	\end{align}
	
	Let $\epsilon_0>0$ be a constant such that $\frac{1}{2^{4/3}}\lbr 1-\sqrt{\frac{3\epsilon_0}{2}}\rbr^{2/3} > \frac{1}{2^{2/3}} - \frac{1}{2^{4/3}}$. According to Lemma~\ref{lemma_random_projection}, there exists absolute constants $c\in(0,1)$, $C>0$, such that so long as $n_t\le cd$, there is with probability at least $1-4\exp(-Cd)$, 
	\begin{align}\label{eq:theorem_joint_finetuning_1}
		\frac{1}{2^{4/3}}\lbr 1-\sqrt{\frac{3\epsilon_0}{2}}\rbr^{2/3} > \frac{1}{2^{2/3}} - \frac{1}{2^{4/3}} + \frac{1}{2^{2/3}} \norm{\hat{\phi}_t}^{4/3}_2.
	\end{align}

	Now we prove the upper bound in Equation~\ref{eq:joint_proof_upper_bound} is smaller than the lower bound in Equation~\ref{eq:joint_proof_lower_bound}. This is because 
	\begin{align}
		&\min_\mu \lbr \frac{3\lambda}{2^{2/3}}|\mu|^{2/3} + \frac{2}{3}(1-\alpha)(\mu-1)^2\rbr + \frac{3\lambda}{2^{2/3}} \norm{\hat{\phi}_t}^{4/3}_2\\
		= &\min_\mu \lbr \frac{3\lambda}{2^{2/3}}|\mu|^{2/3} + \frac{2}{3}(1-\alpha)(\mu-1)^2\rbr - 3\lambda(\frac{2^{2/3} - 1}{2^{4/3}}) + \frac{3\lambda}{2^{2/3}} \norm{\hat{\phi}_t}^{4/3}_2 +3\lambda(\frac{1}{2^{2/3}} - \frac{1}{2^{4/3}})\\
		\le & \min_\mu \lbr \frac{3\lambda}{2^{4/3}}|\mu|^{2/3} + \frac{2}{3}(1-\alpha)(\mu-1)^2\rbr + \frac{3\lambda}{2^{2/3}} \norm{\hat{\phi}_t}^{4/3}_2 +3\lambda(\frac{1}{2^{2/3}} - \frac{1}{2^{4/3}})\\
		\le & \min_\mu \lbr \frac{3\lambda}{2^{4/3}}|\mu|^{2/3} + \frac{2}{3}(1-\alpha)(\mu-1)^2\rbr + \frac{3\lambda}{2^{4/3}} \lbr 1-\sqrt{\frac{3\epsilon_0}{2}}\rbr^{2/3},
	\end{align}
	where the first inequality uses that fact that $|\mu|<1$ for the optimal $\mu$, the second inequality is by Equation~\ref{eq:theorem_joint_finetuning_1}. This completes the proof for joint training.
	
Then, we prove the result about fine-tuning. According to Lemma~\ref{lemma_source_solution}, any minimizer $(\hat{\as}, \phipre)$ of  $\Lsource{\lambda}{\as, \W} $ either satisfies $\phipre=0$, or only one $\w{i}$ is non-zero but looks like (up to scaling) $e_j$ for $j\in[k]$. When $\phipre=0$, there is 
    \begin{align}
	\Exp_{x,y\sim \DT} \lbbr \ell(f_{\hat{\theta}_t, \phipre}(x),y )\rbbr = \frac{2}{3} >\frac{10}{27}.
	\end{align}	
	When only one $\w{i}$ is non-zero but looks like $e_j$ for $j\in[k]$, since all the first $k$ dimensions are equivalent for the source task, with probability $1-\frac{1}{k}$, this dimension is $j\ne1$. The target funciton fine-tuned on this $\phipre$ looks like $f_{\hat{\at}, \phipre}(x) = \gamma x_j^2$ for some $\gamma\in\Real$, so there is 
	\begin{align}
		\Exp_{x,y\sim \DT} \lbbr \ell(f_{\hat{\theta}_t, \phipre}(x),y ) \rbbr &= \Exp_{x,y\sim \DT} \lbbr (\gamma x_t^2 - x_1^2)^2 \rbbr\\
		&= \gamma^2 \Exp_{x,y\sim \DT} \lbbr x_t^4\rbbr -2\gamma\Exp_{x,y\sim \DT} \lbbr x_t^2 x_1^2\rbbr + \Exp_{x,y\sim \DT} \lbbr x_1^4\rbbr\\
		&= \frac{2}{3} \gamma^2 - \frac{8}{9}\gamma +\frac{2}{3} \ge \frac{10}{27}.
	\end{align}
	Combining these two possibilities finishes the proof for fine-tuning. Finnaly, setting $\epsilon = min\{\epsilon_0, \frac{10}{27}\}$ finishes the proof of Theorem~\ref{theorem_joint_finetuning}.
\end{proof}

\subsection{Proof of Theorem~\ref{theorem_meta_transfer}}

\begin{lemma}\label{lemma_source_solution}
	Define the source loss as 
	\begin{align}
	\Lsource{\lambda}{\as, \W} = \LossD{ \DS}{\as, \W}  + \lambda \lbr\norm{\as}^2 + \norm{\W}_F^2\rbr\nonumber.
	\end{align}
	Then, for any $\lambda>0$, any minimizer of $\Lsource{\lambda}{\as, \W}$ is one of the following cases:
	\begin{enumerate}[label=(\roman*)]
		\item $\as=0$ and $\W=0$.
		\item  for one $i\in[\df]$, $\as_i >0$, $\w{i} = \pm (\sqrt{2} \as_i)\cdot e_j$ for some $j\le k$; for all other  $i\in[\df]$, $|\as_i| = \norm{\w{i}} = 0$.
	\end{enumerate}
	Furthermore, when $0<\lambda<0.1$, all the minimizers look like (ii).
\end{lemma}

\begin{proof}[Proof of lemma~\ref{lemma_source_solution}]
	Define $d\times d$ matrix 
	\begin{align}
	A = \sum_{i=1}^{\df} \as_i \w{i}\w{i}^\top. 
	\end{align}
	
	Define $x_{[1:k]}$ and $x_{[k+1:d]}$ be the first $k$ and last $d-k$ dimensions of $x$. $A_{k,k}$, $A_{k, \bar{k}}$ and $A_{\bar{k}, \bar{k}}$ be $k\times k$, $k\times (d-k)$ and $(d-k)\times (d-k)$ matrices that correspond to the upper left, upper right and lower right part of $A$. For a random vector $x$ where the first $k$ dimensions are uniformly independently from $\{\pm 1\}$, the last $d-k$ dimensions are uniformly indepdently from $\{0, \pm1\}$, define random variables
	$A_1 = x_{[1:k]}^\top A_{k,k} x_{[1:k]}$, $A_2 = x_{[k+1:d]}^\top A_{\bar{k},\bar{k}} x_{[k+1:d]}$. (Note that  $x$ is defined on a different distribution than $\DS$.)
	
	The loss part of $\Lsource{\lambda}{\as, \W} $ can be lower bounded by:
	\begin{align}
	&  \Exp_{x,y\sim \DS}\lbbr\ell(f_{\as, \W}(x),y)\rbbr \\
	= &\Exp_{x,y\sim \DS} \lbbr \lbr x_{[1:k]}^\top A_{k,k} x_{[1:k]} + 2x_{[1:k]}^\top A_{k,\bar{k}} x_{[k+1:d]}  +   x_{[k+1:d]}^\top A_{\bar{k},\bar{k}} x_{[k+1:d]} - y\rbr^2 \rbbr\\
	\ge & \Exp_{x,y\sim \DS} \lbbr \lbr x_{[1:k]}^\top A_{k,k} x_{[1:k]} +  x_{[k+1:d]}^\top A_{\bar{k},\bar{k}} x_{[k+1:d]} - y\rbr^2 \rbbr\\
	= & \frac{2}{3}\Exp\lbbr \lbr A_1 - 1 \rbr^2 \rbbr + \frac{4}{3}\Exp \lbbr  (A_1-1) A_2\rbbr + \Exp\lbbr A_2^2\rbbr\\
	= & \frac{2}{3} \Exp\lbbr\lbr A_1 +  A_2 - 1 \rbr^2\rbbr + \frac{1}{3} \Exp\lbbr A_2^2\rbbr.
	\end{align}
	The inequality is because 
	\begin{align}
		\Exp_{x,y\sim \DS}\lbbr \lbr x_{[1:k]}^\top A_{k,k} x_{[1:k]} \rbr\lbr x_{[1:k]}^\top A_{k,\bar{k}} x_{[k+1:d]} \rbr \rbbr = 0,
	\end{align}
	\begin{align}
\Exp_{x,y\sim \DS}\lbbr \lbr  x_{[k+1:d]}^\top A_{\bar{k},\bar{k}} x_{[k+1:d]}  \rbr\lbr x_{[1:k]}^\top A_{k,\bar{k}} x_{[k+1:d]} \rbr \rbbr = 0,
\end{align}
	\begin{align}
\Exp_{x,y\sim \DS}\lbbr \lbr x_{[1:k]}^\top A_{k,\bar{k}} x_{[k+1:d]} \rbr y  \rbbr = 0.
\end{align}
The inequality is equality if and only if $A_{k,\bar{k}} = 0$.
	
	The regularizer part of $\Lsource{\lambda}{\as, \W} $ can be lower bounded by:
	\begin{align}
	\sum_{i=1}^{\df} \lbr \lbr\as_i\rbr^2 + \norm{\w{i}}^2 \rbr 
	\ge& \frac{3}{2^{2/3}} \sum_{i=1}^{\df} \lbr |\as_i| \cdot \norm{\w{i}}^2\rbr^{2/3}\\
	\ge& \frac{3}{2^{2/3}}  \lbr \sum_{i=1}^{\df} |\as_i| \cdot \norm{\w{i}}^2 \rbr^{2/3}\\
	\ge& \frac{3}{2^{2/3}}  \lbr \sum_{i=1}^{\df} |A_{[i,i]}| \rbr^{2/3},
	\end{align}
	where the first inequality is AM-GM inequality, the second inequality is by concavity of $(\cdot)^{2/3}$. The third inequality is because for diagonal matrix $D$ that has $1$ at $(i,i)$ if $A_{[i,i]}\ge0$, $-1$ at $(i,i)$ if $A_{[i,i]}< 0$, we have
	\begin{align}
	\sum_{i=1}^{\df} |A_{[i,i]}|  = tr(AD) =  \sum_{i=1}^{\df} \as_i \w{i}^\top D \w{i} \le \sum_{i=1}^{\df} |\as_i| \cdot \norm{\w{i}}^2 .
	\end{align}
	All the inequalities are equality if and only if $\lbr\as_i\rbr^2 =\frac{1}{2} \norm{\w{i}}^2>0$ for at most one $i\in[\df]$, and for all other $i\in[\df]$ there is $|\as_i| =\norm{\w{i}}=0$.
	
	Combining the two parts gives a lower bound for $\Lsource{\lambda}{\as, \W} $:
	\begin{align}
	\Lsource{\lambda}{\as, \W} \ge & \frac{3\lambda}{2^{2/3}}  \lbr \sum_{i=1}^{\df} |A_{[i,i]}| \rbr^{2/3} +  \frac{2}{3} \Exp\lbbr\lbr A_1 +  A_2 - 1 \rbr^2\rbbr + \frac{1}{3} \Exp\lbbr A_2^2\rbbr\\
	\ge & \frac{3\lambda}{2^{2/3}}  \lbr  |tr(A_{k,k} + \frac{2}{3} A_{\bar{k}, \bar{k}})| \rbr^{2/3} +  \frac{2}{3} \lbr \Exp\lbbr A_1 + A_2 \rbbr - 1 \rbr^2,
	\end{align} 
	where both inequalitites are equality if and only if $A_{\bar{k}, \bar{k}} = 0$ (therefore $A_2=0$) and $Var\lbbr A_1\rbbr = 0$ (therefore $A_{k,k}$ is diagonal by Lemma~\ref{lemma_no_variance}).

	Notice that $\Exp[A_1 + A_2] = tr(A_{k,k} + \frac{2}{3} A_{\bar{k}, \bar{k}})$, the above lower bound is further minimized when $\Exp[A_1]=\mu^*$ where $\mu^*$ is the minimizer of function $L(\mu) =\frac{3\lambda}{2^{2/3}}  \lbr |\mu| \rbr^{2/3} + \frac{2}{3} \lbr \mu- 1\rbr^2$.

	To see when this lower bound is achieved, we combine all the conditions for the inequalities to be equality. When $\mu^*=0$, this lower bound is achieved if and only if $\as = 0$ and $\W = 0$. 
	When $\mu^*>0$, this lower bound is only achieved when the solution look like this: for one $i\in[\df]$, $\as_i=(\frac{\mu^*}{2})^{1/3}$, $\w{i} = \pm (\sqrt{2} \as_i)\cdot e_j$ for some $j\le k$; for all other  $i\in[\df]$, $|\as_i| = \norm{\w{i}} = 0$.
	
	Obviously, there is either $\mu^* = 0$ or $\mu^* >0$. Also, when $\lambda<0.1$, the minimizer $\mu^*$ of  $L(\mu)$ is strictly larger than $0$ (since $L(1)<L(0)$). So this completes the proof.
\end{proof}

\begin{lemma}\label{lemma_no_variance}
	Let $M\in \Real^{k\times k}$ be a symmetric matrix, $x\in\Real^k$ is a random vector where each dimension is indepedently uniformly from $\{\pm 1\}$. Then, $\Var[x^\top M x]=0$ if and only if $M$ is a diagonal matrix.
\end{lemma}
\begin{proof}[Proof of Lemma~\ref{lemma_no_variance}]
	In one direction, when $M$ is diagonal matrix, obviously $\Var[x^\top M x]=0$. In the other direction, when $\Var[x^\top M x]=0$, there has to be $x^\top M x$ be the same for all $x\in \{\pm 1\}^k$. For any $i\ne j$, let  $x^{(1)} = 1-2e_i - 2e_j$, $x^{(2)} = 1$, $x^{(3)} = 1-2e_i$, $x^{(4)} = 1-2e_j$. Then the $(i,j)$ element of $M$ is $\frac{1}{8} \lbr {x^{(1)} }^\top M  x^{(1)} + {x^{(2)} }^\top M  x^{(2)} - {x^{(3)} }^\top M  x^{(3)}- {x^{(4)} }^\top M  x^{(4)}\rbr$, which is $0$. So $M$ has to be a diagonal matrix.
\end{proof}

\begin{proof}[Proof of Theorem~\ref{theorem_meta_transfer}.]	
	Define the source loss as in Lemma~\ref{lemma_source_solution}, then we have 
	\begin{align}
	\Lmeta{\lambda}{\as, \W} = \Lsource{\lambda}{\as, \W} + \Exp \lbbr \LossD{\DTb}{\thetahatphi, \W} \rbbr.
	\end{align}
	By Lemma~\ref{lemma_source_solution}, the source loss $\Lsource{\lambda}{\as, \W} $ is minimized by a set of solutions that look like this: for one $i\in[\df]$, $\as_i >0$, $\w{i} = \pm (\sqrt{2} \as_i)\cdot e_j$ for some $j\le k$; for all other  $i\in[\df]$, $|\as_i| = \norm{\w{i}} = 0$. 	

	When $j=1$, the only feature in $\W$ is $e_1$.  When $\nt\ge18\log\frac{2}{\xi}$, according to Chernoff bound, with probability at least $1-\frac{\xi}{2}$ there is strictly less than half of the data satisfy $x_1=0$. Therefore, any $\DTa$ contains data with $x_1\ne 0$, and the only target head that fits $\DTa$ has to recover the ground truth. Hence there is $\Exp \lbbr \LossD{\DTb}{\thetahatphi, \W} \rbbr = 0$. 
	
	When $j\ne 1$, the only feature is $e_j$. This feature can be used to fit the target data if and only if either $x_{i [j]}^2=x_{i[1]}^2$ for all target data $x_i$, or $x_{i[1]}=0$ for all $x_i$. Since there are at most $k-1$ possible $j$, by union bound we know the probability of any of these happens for any $j\ne 1$  is at most $k (\frac{2}{3})^{n_t}$. Hence, when $n_t\ge 3\log \frac{2k}{\xi}$, the probability of any $e_j$ fits the target data is smaller than $\frac{\xi}{2}$.
	Therefore, with probabiltiy $1-\frac{\xi}{2}$,  $\Exp \lbbr \LossD{\DTb}{\thetahatphi, \W} \rbbr > 0$ for any $j\ne 1$. 
	
	So with probability at least $1-\xi$, the only minimizer of $\Lmeta{\lambda}{\as, \W}$ is the subset of minimizers of $\Lsource{\lambda}{\as, \W} $ with feature $e_1$, and with this $\W$ and any random $\DTa$, the only $\at(\DTa, \W)$ that fits the target recovers the ground truth, i.e.,  $\Exp_{x,y\sim \DT} \lbbr \losssub{\at(\DTa, \W), \W}{x,y} \rbbr = 0$.
	
\end{proof}

\end{document}